\documentclass[a4paper,fleqn]{cas-dc}

\usepackage[numbers]{natbib}

\usepackage{amsmath}
\usepackage{amsthm}
\usepackage{amssymb}
\usepackage{algorithm}

\newtheorem{theorem}{Theorem}
\newtheorem{lemma}{Lemma}

\newtheorem{definition}{Definition}
\newtheorem{assumption}{Assumption}

\usepackage{algpseudocode} 
\usepackage{hyperref}  
\usepackage{pgfplots}
\usepackage{float}
\usepackage{booktabs}
\usepackage{multirow}
\usepackage{threeparttable}
\usepackage{subcaption}

\def\tsc#1{\csdef{#1}{\textsc{\lowercase{#1}}\xspace}}
\tsc{WGM}
\tsc{QE}
\tsc{EP}
\tsc{PMS}
\tsc{BEC}
\tsc{DE}

\begin{document}
\renewcommand{\dbltopfraction}{0.95}   
\renewcommand{\dblfloatpagefraction}{0.9}
\renewcommand{\topfraction}{0.95}      
\renewcommand{\bottomfraction}{0.95}   
\renewcommand{\textfraction}{0.05}     
\setcounter{dbltopnumber}{3}           
\setcounter{topnumber}{2}

\let\WriteBookmarks\relax
\def\floatpagepagefraction{1}
\def\textpagefraction{.001}
\shorttitle{Fine-grained Distributed Backdoor Attacks in Federated Learning}
\shortauthors{Wang Jian et~al.}

\title [mode = title]{Fine-grained Distributed Backdoor Attacks in Federated Learning}                      

\author[1,2]{Jian Wang}[style=chinese, prefix=Dr, orcid=0000-0001-7511-2910]
\cormark[1]
\fnmark[1]
\ead{jenseWang@outlook.com}
\credit{Conceptualization, Methodology, Writing - Original Draft}

\author[3]{Hong Shen}[style=chinese, prefix=Prof]
\fnmark[2]

\credit{Conceptualization, Methodology, Supervision}

\author[1]{Wei Ke}[style=chinese, prefix=Prof]
\fnmark[3]

\credit{Project administration, Supervision}

\author[4]{Xue Hua Liu}[style=chinese, prefix=Prof]

\credit{Data Curation, Resources}

\cortext[1]{Corresponding author: jenseWang@outlook.com}

\affiliation[1]{organization={Faculty of Applied Sciences, Macao Polytechnic University},
	addressline={R. de Luís Gonzaga Gomes},
	city={Macao},
	postcode={999078},
	country={China}}

\affiliation[2]{organization={Faculty of Cyberspace Security, Guangzhou University of Software},
	addressline={No.548 Guangcong South Road, High-tech Industrial Park, Conghua Economic Development Zone},
	city={Guangzhou},
	postcode={501900},
	country={China}}
\affiliation[3]{organization={School of Engineering and Technology, Central Queensland University},
	addressline={Bruce Highway, Norman Gardens},
	city={Rockhampton},
	state={Queensland},
	postcode={4702},
	country={Australia}}

\affiliation[4]{organization={Faculty of  Software and Artificial Intelligence, Guangzhou University of Software},
	addressline={No.548 Guangcong South Road, High-tech Industrial Park, Conghua Economic Development Zone},
	city={Guangzhou},
	postcode={501900},
	country={China}}

\fntext[fn1]{Supported by the Guangdong Provincial Department of Education (Grant No. 2024KTSCX133).}
\fntext[fn2]{Supported by the Queensland Department of Environment and Science Quantum Challenges 2032 Program (Grant No. Q2032001).}
\fntext[fn3]{Supported by the Science and Technology Development Fund, Macao SAR (File No. 0015/2023/RIA1)}


\begin{highlights}
	\item Proposes a dynamic trigger generation method based on Canny edge features and Laplacian noise injection for covert and adaptive backdoor embedding.  
	\item Designs an RGB channel decomposition strategy that enables distributed trigger injection among multiple malicious clients, enhancing stealthiness.  
	\item Introduces embedding vector optimization with contrastive learning and random projection hashing to amplify poisoned sample impact while reducing poisoning ratio.  
	\item Derives, under the stated analytical assumptions, the poisoning-ratio relation $1/(1+\gamma\alpha)$; on CIFAR-10, piecewise-linear estimates for target ASRs of 70\%--90\% indicate 37.4\%--48.4\% fewer poisoned samples than DBA.  
	\item Demonstrates through extensive experiments that the framework maintains high attack success rates under non-IID settings and effectively bypasses mainstream defense mechanisms.  
\end{highlights}

\begin{abstract}
Federated learning, as a privacy-preserving distributed machine learning paradigm, faces significant threats from backdoor attacks. Compared to centralized attacks, distributed backdoor attacks are more harmful but require more poisoned samples to compensate for the loss of trigger strength due to decomposition. Fixed trigger patterns are also easily detected by robust aggregation algorithms, increasing the risk of attack exposure.  To address these challenges, we propose a fine-grained distributed backdoor attack framework (FDBA). This framework uses dynamic trigger generation and embedding vector optimization to perform attacks with fewer poisoned samples. First, we design a dynamic trigger generation method based on image edge structures using the Canny algorithm to extract edge features, which are then injected with Laplacian noise. RGB channel decomposition is applied for covert adaptation of the distributed trigger, reducing detection chances. Second, we introduce an embedding vector contrastive learning strategy that forces poisoned samples to approach the target class center in the feature space, enhancing attack effectiveness.
On CIFAR-10, piecewise-linear estimates for target ASRs between 70\% and 90\% show that FDBA reduces the required poisoning ratio by 37.4\%--48.4\% compared with DBA. In non-independent and identically distributed (Non-IID) scenarios, FDBA retains 84.7\% of its IID attack performance under extreme heterogeneity, whereas DBA drops to 73.5\%, and the framework successfully bypasses mainstream defense mechanisms.

This study offers new insights into federated learning security and emphasizes the potential threats and defense challenges posed by fine-grained distributed attacks.
\end{abstract}


\maketitle

\section{Introduction}

Federated learning enables distributed machine learning by allowing multiple clients to train models locally and upload updates to a central server for aggregation. This approach effectively leverages multiple data sources while protecting data privacy, making it widely applicable in fields like medical diagnostics, financial analysis, and smart devices. However, its distributed nature also presents opportunities for attackers, making federated learning vulnerable to unique security threats, particularly data poisoning attacks \cite{kairouz2021advances}. Data poisoning is a type of attack in which adversaries inject malicious data into the training set to disrupt the model's learning process, ultimately leading to incorrect predictions \cite{goldblumDatasetSecurityMachine2021}. A backdoor attack, a covert and dangerous form of data poisoning, involves embedding a trigger into the poisoned samples, allowing the model to produce a preset output under specific input conditions while maintaining normal predictions for other inputs \cite{liRethinkingTriggerBackdoor2021}.
In federated learning systems, distributed backdoor attacks coordinate multiple malicious clients by decomposing a global trigger into local sub-triggers \cite{xieDBADISTRIBUTEDBACKDOOR2020}. This decomposition creates a practical trade-off: no single client carries the complete trigger, but each local component has weaker influence. Therefore, a key challenge is how to design covert distributed triggers while limiting the number of poisoned samples.

To address these critical challenges, we propose the FDBA framework. The core intuition of FDBA is to shift the attack paradigm from static, pixel-level manipulation to dynamic, feature-level geometric optimization. Conceptually, FDBA operates through a synergistic pipeline: First, instead of using fixed visual patterns that are easily distorted by non-IID data distributions, FDBA employs Canny edge detection and Laplacian noise to dynamically generate structure-aware triggers that blend naturally with the host image. Second, to evade anomaly detection mechanisms that scrutinize coordinated malicious updates, we decouple the trigger into RGB color channels and distribute them across different malicious clients. 

Most importantly, to achieve maximum poisoning efficiency, we introduce a contrastive learning-based embedding optimization strategy. By mathematically forcing the feature representations of poisoned samples closer to the target class center in the latent space, we dramatically amplify the gradient impact of each malicious sample during the federated aggregation phase. This geometric amplification allows FDBA to achieve high attack effectiveness even under constrained poisoning budgets, effectively bridging the gap between stealthiness and attack power. 

The main contributions are as follows:
\begin{itemize}
	\item \textbf{Fine-grained Attack Framework}: In response to the challenges posed by existing DBAs that rely on large amounts of poisoned samples and lack stealth, we propose a new fine-grained DBA framework that significantly reduces the required number of poisoned samples through precise trigger placement and optimized embedding strategies.
	\item \textbf{Fine-grained Trigger Design}: A precise trigger generation mechanism based on image edge structures is designed to achieve fine-grained control over trigger placement, addressing the issue of traditional fixed triggers being detectable while enabling targeted poisoning strategies.
	\item \textbf{Embedding Vector Optimization}: We propose an embedding vector optimization method that magnifies the impact of poisoned samples on the global model, enabling effective attacks on federated learning systems with fewer poisoned samples.
	\item  \textbf{Analytical Characterization:}\\ Under the cumulative-influence model and the assumptions stated in Appendix~\ref{sec:Proof}, we derive the poisoning-ratio relation $\rho = (1 + \gamma\alpha)^{-1}$, where $\gamma$ represents the contrastive learning gain and $\alpha$ the edge structure preservation factor.
\end{itemize}

\section{Related Work}
Research on backdoor attacks and defenses in federated learning has developed into a dynamic interplay between offensive and defensive strategies. The technical evolution can be summarized as an enhancement of attack stealthiness and an adaptive progression of defense systems. This section provides a systematic review from three perspectives: the evolution of attack methods, stealth-enhancement techniques, and the development of defense mechanisms.

\subsection{Evolution of Attack Methods}
Early backdoor research focused on centralized settings. BadNets demonstrated that a model can learn a trigger-dependent behavior while retaining normal accuracy on clean inputs \cite{gu2019badnets}, and Chen et al. studied targeted backdoor injection through poisoned training data \cite{chen2017targeted}. These settings assume influence over centralized training data and therefore differ from federated learning. Bagdasaryan et al. introduced model-replacement attacks in federated learning by scaling a malicious client update to replace the aggregated global model \cite{byzantinegradientdescentHowBackdoorFederated2020}. Bhagoji et al. likewise analyzed model-poisoning strategies designed to remain effective under federated aggregation \cite{bhagoji2019analyzing}.

Xie et al. proposed Distributed Backdoor Attacks (DBA), decomposing a complete trigger into complementary sub-patterns injected by multiple malicious clients \cite{xieDBADISTRIBUTEDBACKDOOR2020}. Wang et al. subsequently studied distributed attacks based on dynamically generated global triggers that are decomposed into client-side sub-triggers \cite{wang2024distributed}.

Recent research has extended into multimodal and cross-task scenarios: Zhang et al. implanted semantic backdoors in NLP models, triggering misclassification with specific keyword combinations \cite{zhang2021trojaning}.

\subsection{Stealth-Enhancement Techniques}
Traditional pixel-block triggers can be visually conspicuous, motivating stealth-oriented trigger designs. Li et al. combined steganography and regularization to construct visually inconspicuous backdoors \cite{li2020invisible}. Doan et al. jointly modified the input and latent representations to improve imperceptibility \cite{doan2021backdoor}. In federated learning, Gong et al. proposed coordinated attacks with model-dependent triggers \cite{gongCoordinatedBackdoorAttacks2022}.

Sun et al. formulated data poisoning in federated machine learning as an optimization problem that accounts for the federated training process \cite{sunDataPoisoningAttacks2020}.

Dynamic backdoor attacks generate input-dependent triggers rather than using a single fixed pattern \cite{salem2022dynamic}. In federated settings, Chen et al. studied backdoor attacks against federated meta-learning \cite{chen2020backdoor}, while Rieger et al. proposed DeepSight, which inspects model updates to mitigate federated backdoors \cite{rieger2022deepsight}.

\subsection{Evolution of defense mechanisms }
The first generation of defenses focused on robust aggregation: Blanchard's Krum selects updates closest to the majority cluster \cite{blanchard2017machine}, while Yin et al. proposed coordinate-wise trimmed aggregation methods with theoretical Byzantine-robustness guarantees \cite{yin2018byzantine}. FLCert instead provides certified robustness by partitioning clients into groups, learning multiple global models, and aggregating their predictions by majority vote \cite{cao2022flcert}. FLAME clusters and clips model updates before adding noise to the aggregated update \cite{nguyen2022flame}.

Feature-based detection methods have also evolved rapidly. Neural Cleanse reverse-engineers candidate triggers and detects anomalously small trigger patterns \cite{wang2019neural}. Meta Neural Analysis trains a meta-classifier to distinguish Trojaned from clean models \cite{xu2021detecting}. In federated learning, BAFFLE uses client feedback on the global model to detect backdoor behavior without requiring access to client training data \cite{andreina2021baffle}.

Furthermore, some study has considered how to repair the poisoned model. Model repair techniques include Fine-Pruning, which prunes backdoor-related neurons but leads to a decrease in normal accuracy \cite{liu2018fine}. Wu’s adversarial pruning identified vulnerable neurons using adversarial samples, achieving higher backdoor removal rates on CIFAR-10 \cite{wu2021adversarial}. In federated learning, Huang’s Lockdown built an isolation subspace training mechanism, using orthogonal projections to separate malicious parameters, maintaining high normal accuracy while eliminating most backdoors on ImageNet tasks \cite{huang2024lockdown}.

\subsection{Current Limitations and Research Gaps}

Table~\ref{tab:backdoor_evolution} provides a comprehensive overview of the evolution of backdoor attacks and defenses in federated learning, systematically categorizing the key innovations, performance achievements, and inherent limitations across different research directions.

\begin{table*}
	\centering
	\caption{Summary of backdoor attacks and defenses in federated learning}
	\label{tab:backdoor_evolution}
	\begin{threeparttable}
		\begin{tabular}{@{}lllll@{}}
			\toprule
			\textbf{Category} & \textbf{Method} & \textbf{Key feature} & \textbf{Performance} & \textbf{Main limitation} \\
			\midrule
			\multirow{3}{*}{\begin{tabular}[c]{@{}l@{}}	\textbf{Attack}\\ 	\textbf{evolution}\end{tabular}} 
			& BadNets~\cite{gu2019badnets} & Pixel trigger & Triggered misclassification & Centralized setting \\
			& Model replacement~\cite{byzantinegradientdescentHowBackdoorFederated2020} & Weight amplification & High ASR & Easily detected \\
			& DBA~\cite{xieDBADISTRIBUTEDBACKDOOR2020} & Trigger decomposition & Multi-client coordination & Requires coordinated clients \\
			\midrule
			\multirow{3}{*}{\begin{tabular}[c]{@{}l@{}}\textbf{Stealth}\\ \textbf{enhancement}\end{tabular}} 
			& Invisible backdoor~\cite{li2020invisible} & Steganography + regularization & Visually subtle triggers & Centralized setting \\
			& Federated poisoning~\cite{sunDataPoisoningAttacks2020} & Optimization-based attack & FL-aware objective & Optimization cost \\
			& Dynamic backdoor~\cite{salem2022dynamic} & Input-dependent trigger & Trigger diversity & Centralized setting \\
			\midrule
			\multirow{4}{*}{\begin{tabular}[c]{@{}l@{}}\textbf{Defense}\\ \textbf{mechanisms}\end{tabular}} 
			& Krum~\cite{blanchard2017machine} & Geometric update selection & Byzantine robustness & Honest-majority assumptions \\
			& FLCert~\cite{cao2022flcert} & Group ensemble + voting & Certified robustness & Multiple model training \\
			& BAFFLE~\cite{andreina2021baffle} & Client feedback & Server-data-free detection & Extra client evaluation \\
			& Lockdown~\cite{huang2024lockdown} & Subspace isolation & High normal accuracy & Complex implementation \\
			\bottomrule
		\end{tabular}
		\begin{tablenotes}
			\item ASR, attack success rate; Non-IID, non-independent and identically distributed data.
		\end{tablenotes}
	\end{threeparttable}
\end{table*}

As illustrated in Table~\ref{tab:backdoor_evolution}, the research landscape reveals a continuous arms race between attack sophistication and defense advancement, with each generation of methods addressing limitations of their predecessors while introducing new challenges. The evolution from centralized attacks like BadNets to distributed approaches such as DBA demonstrates the field's adaptation to federated learning constraints, yet significant gaps remain unresolved.

Current study on distributed backdoor attacks faces significant challenges, primarily manifesting in two critical dimensions: achieving both stealthiness and effectiveness, and addressing the need for fine-grained control over attack precision and targeting. The fundamental tension between these objectives creates a persistent dilemma where high-stealth attacks invariably sacrifice effectiveness, while highly effective attacks rely on elevated poisoning rates that substantially increase detection risks.

Traditional distributed attacks exemplify this challenge through their trigger decomposition strategies. Methods like DBA decompose complete triggers into complementary sub-patterns distributed across malicious clients (e.g., each client injecting localized features from one quadrant of the image) to evade detection. However, this decomposition inherently weakens the global trigger's potency and semantic coherence, resulting in diminished attack effectiveness. The distributed nature creates a dilution effect that compromises trigger integrity while failing to achieve the desired stealth-effectiveness balance.

More critically, existing solutions demonstrate systematic failures in low-poisoning scenarios due to their reliance on gradient alignment strategies that demand extremely precise synchronization of poisoned sample gradients. This synchronization requirement, as demonstrated by methods constraining updates within ±15° of benign gradients, creates operational limitations that force attackers into suboptimal choices: maintaining low poisoning rates with significantly reduced effectiveness or increasing poisoning rates while elevating detection risks.

The gradient alignment constraint becomes particularly problematic as it requires maintaining poisoned gradients within narrow angular bounds relative to benign gradients. This precision requirement intensifies with decreasing malicious client numbers and increasing data heterogeneity across federated participants. Consequently, attackers are compelled to increase poisoning rates to achieve necessary gradient coordination, thereby compromising fundamental stealth objectives and significantly raising the risk of exposure and detection.

Meanwhile, the continuous development of defense mechanisms has progressively constrained the operational space for effective attacks. The evolution from simple gradient filtering methods like Krum to sophisticated behavioral analysis systems like FLAME and BAFFLE demonstrates the defensive community's rapid adaptation. Modern defenses incorporate temporal patterns, multi-dimensional anomaly detection, and advanced statistical analysis, rendering traditional attack strategies increasingly obsolete and forcing attackers to develop more complex circumvention strategies.

The data heterogeneity inherent in federated learning environments introduces additional complexity layers. Non-IID data distributions across clients challenge both attack consistency and defense effectiveness, creating scenarios where attack success varies dramatically across different federated participants. Current stealth enhancement techniques, while addressing some aspects of this challenge through adaptive mechanisms, often introduce computational overhead and implementation complexity that limit their practical applicability.

In addition, previous works such as DBA~\cite{xieDBADISTRIBUTEDBACKDOOR2020} and Neurotoxin~\cite{zhang2022neurotoxin} provide empirical evidence of the effectiveness of these attacks, yet they lack a strict theoretical analysis of poisoning efficiency. Our work fills this gap by establishing a formal mathematical framework that quantifies the poisoning reduction enabled by contrastive learning and edge-based triggers.

Therefore, there exists an urgent need for a new distributed attack paradigm that can simultaneously maintain high attack success rates under low-poisoning constraints, demonstrate robustness to heterogeneous data distributions characteristic of federated learning environments, and effectively bypass increasingly sophisticated mainstream defense detection mechanisms. Such a paradigm would require fundamental innovations in trigger design methodologies, gradient coordination strategies, and attack orchestration frameworks that transcend the current limitations outlined in Table~\ref{tab:backdoor_evolution} while addressing the core challenges of stealth-effectiveness trade-offs and synchronization requirements in distributed attack scenarios.

\section{Preliminaries}
The mathematical notation used throughout this paper is summarized in Table~\ref{tab:notation} for reference.

\begin{table}[htbp]
	\centering
	\caption{Mathematical Notation}
	\label{tab:notation}
	\begin{tabular}{cl}
		\toprule
		\textbf{Symbol} & \textbf{Definition} \\
		\midrule
		$K$ & Number of clients in federated learning \\
		$K_m$ & Set of malicious clients \\
		$D_i$ & Local dataset of client $i$ \\
		$|D_i|$ & Size of local dataset $D_i$ \\
		$w$ & Global model parameters \\
		$w_t$ & Global model parameters at round $t$ \\
		$\Delta w_i$ & Model update from client $i$ \\
		$f(w, x)$ & Global model with parameters $w$ and input $x$ \\
		$\ell(\cdot)$ & Loss function \\
		$T(\cdot)$ & Trigger function \\
		$T'(x)$ & Edge-based trigger for sample $x$ \\
		$y^*$ & Target attack label \\
		$p$ & Poisoning ratio (fraction of poisoned samples) \\
		$\rho$ & Poisoning ratio reduction factor $\rho<1$ \\
		$\gamma$ & Contrastive learning amplification gain \\
		$\alpha$ & Edge structure preservation factor \\
		$\lambda$ & Embedding strength control parameter \\
		$\mu, b$ & Location and scale of the Laplace noise \\
		$e_i$ & Embedding vector of sample $i$ \\
		$e_p, e_n$ & Promoter and distractor embedding vectors \\
		$[\cdot]_+$ & ReLU function: $\max(\cdot, 0)$ \\
		$\mathbb{E}_{(x,y)\sim D}[\cdot]$ & Expectation over distribution $D$ \\
		$\mathbb{P}[\cdot]$ & Probability measure \\
		$\|\cdot\|$ & Euclidean norm \\
		$\|\cdot\|_F$ & Frobenius norm \\
		$G_x, G_y$ & Image gradients along the $x$ and $y$ axes \\
		$M(x,y)$ & Edge magnitude at pixel $(x,y)$ \\
		$T_H, T_L$ & High and low thresholds for edge detection \\
		$\sigma$ & Gaussian smoothing parameter \\
		$R$ & Random projection matrix \\
		$h(\cdot)$ & Random projection hash function \\
		\bottomrule
	\end{tabular}
\end{table}

\subsection{Federated Learning}
Federated Learning (FL) is a distributed machine learning approach that enables multiple clients to train models locally and upload parameter updates to a central server for aggregation, avoiding direct sharing of raw data and effectively protecting data privacy \cite{smithFederatedMultiTaskLearning2017}. In a typical FL framework, assume there are $K$ clients participating in the training, and each client $i$ holds a local dataset $D_{i}=\{(x_{j},y_{j})\}_{j=1}^{|D_{i}|}$. The global model optimization objective is:

\begin{equation}
	\min_{w}F(w)=\frac{1}{K}\sum_{i=1}^{K}F_{i}(w),
\end{equation}

where $F_{i}(w)=\frac{1}{|D_{i}|}\sum_{(x_{j},y_{j})\in D_{i}}\ell(w;x_{j},y_{j})$ is the local loss function of client $i$, and $\ell$ is the loss function. FL performs iterative training through the following three steps:

\begin{enumerate}
	\item Global Model Distribution: The server sends the global model parameters $w_{t}$ to the clients.
	\item Local Training: Clients update the model $\Delta w_{i}$ using their local dataset.
	\item Global Aggregation: The server aggregates the client updates to compute new global model parameters:
\end{enumerate}
$$
w_{t+1}=w_{t}+\eta\cdot\text{Aggregate}(\{\Delta w_{i}\}_{i=1}^{K}),
$$
where $\eta$ is the learning rate and $\text{Aggregate}$ is the aggregation operation (such as weighted averaging). FL’s distributed nature enhances data privacy but also makes it vulnerable to backdoor attacks and other covert threats.

\subsection{Backdoor Attacks}

The core of a backdoor attack is to stealthily implant a trigger in the global model so that the model outputs a pre-determined result under specific input conditions, while maintaining good performance on normal inputs \cite{gu2019badnets}. The goal can be formalized as:
\begin{enumerate}
	\item \textbf{Normal Performance}: The global model $f(w,x)$ retains high performance on normal input $x$:
	\begin{equation}
		\begin{split}
			\mathbb{E}_{(x,y)\sim D_{\text{benign}}}\ell(f(w,x),y) 
			&\approx \mathbb{E}_{(x,y)\sim D_{\text{benign}}}\ell(f(w_{\text{clean}},x),y)
		\end{split}
	\end{equation}
	where $D_{\text{benign}}$ is the benign data distribution, and $w_{\text{clean}}$ is the unpoisoned model parameters.
	\item \textbf{Backdoor Trigger}: When the input is injected with a trigger $T(x)$, i.e., $x'=T(x)$, the model outputs the attacker's target label $y^{*}$:
	$$
	f(w,x')=y^{*}.
	$$
\end{enumerate}
In the federated learning scenario, backdoor attacks are typically executed through poisoned sample injection and malicious updates uploaded to the server. Poisoned sample injection occurs when malicious clients inject poisoned samples $(x', y^{*})$ into their local dataset and participate in model training. Malicious updates involve manipulating the uploaded model updates $\Delta w_{i}$ to covertly influence the global model. Distributed backdoor attacks further enhance stealth and flexibility by utilizing multiple malicious clients to collaborate in injecting poisoned samples or triggers.

\subsection{System and Threat Model}
\label{subsec:threat_model}

To rigorously evaluate the proposed Federated Dynamic Backdoor Attack (FDBA), we establish a precise system and threat model detailing the roles, knowledge, and capabilities within the federated learning environment.

\textbf{System Model:} We consider a standard cross-device Federated Learning (FL) system consisting of two primary roles: a central parameter server and $N$ participating clients. In each communication round, the server distributes the global model to a randomly selected subset of clients. These clients perform local training using their private datasets and submit the updated model gradients back to the server. The server then aggregates these local updates (e.g., via FedAvg or robust aggregation rules) to update the global model.

\textbf{Threat Model (Knowledge and Capabilities):} 
\begin{itemize}
	\item \textbf{Attacker Role:} We assume a fraction of the participating clients are compromised and act as malicious attackers. These clients can collude to achieve a shared objective.
	\item \textbf{Knowledge:} The malicious clients possess \textit{white-box} access to their own local training data, local training processes, and the global model broadcasted by the server in each round. However, they have \textit{black-box} access to the benign clients and the server's exact aggregation mechanism.
	\item \textbf{Capabilities:} The attackers can arbitrarily manipulate their local datasets and optimize their local model parameters before uploading them to the server. Crucially, the attackers cannot compromise the central server, nor can they intercept or alter the communications of benign clients.
\end{itemize}

\subsection{Fine-grained Poisoning}
In federated learning, an attacker can implant a backdoor into the global model through fine-grained poisoning strategies that precisely target specific data regions and model parameters. This involves injecting carefully crafted poisoned samples $\{(x'_{i}, y^{*})\}$ with optimized placement and uploading strategically designed malicious model updates $\Delta w_{\text{malicious}}$. The goal is to achieve effective backdoor implantation through precise control over:

\begin{enumerate}
	\item \textbf{Guarantee Normal Performance}: The global model's performance on normal data remains unchanged or only slightly altered \cite{byzantinegradientdescentHowBackdoorFederated2020}.
	\item \textbf{Successful Backdoor Trigger}: The global model outputs the attacker's target label $y^{*}$ when the trigger input $x'$ is applied \cite{chen2017targeted}.
\end{enumerate}
The formal definition can be represented as:
Given the global model $f(w)$, the number of clients $K$, the malicious client subset $K_{\text{malicious}} \subset K$, and benign data distribution $D_{\text{benign}}$, and a trigger generation function $T(\cdot)$, the goal is first to maximize normal performance, i.e.,

\begin{equation}
	\max_{w}\,\mathbb{E}_{(x,y)\sim D_{\text{benign}}}[\ell(f(w,x),y)],
\end{equation}

while ensuring $\Delta \text{Perf} \leq \epsilon$, where $\Delta \text{Perf}$ represents the change in model performance, and $\epsilon$ is the allowed performance deviation. Secondly, maximize the backdoor success rate, i.e.,

\begin{equation}
	\max_{w}\,\mathbb{P}[f(w,x')=y^{*}], \quad x' = T(x).
\end{equation}

\section{Method}

As shown in the Figure \ref{FIG:Framework of the Few-shot Distributed Backdoor Attack.}, our fine-grained approach primarily involves the following four key steps with precise control mechanisms:
\begin{figure*}
	\centering
	\includegraphics[width=0.9\textwidth]{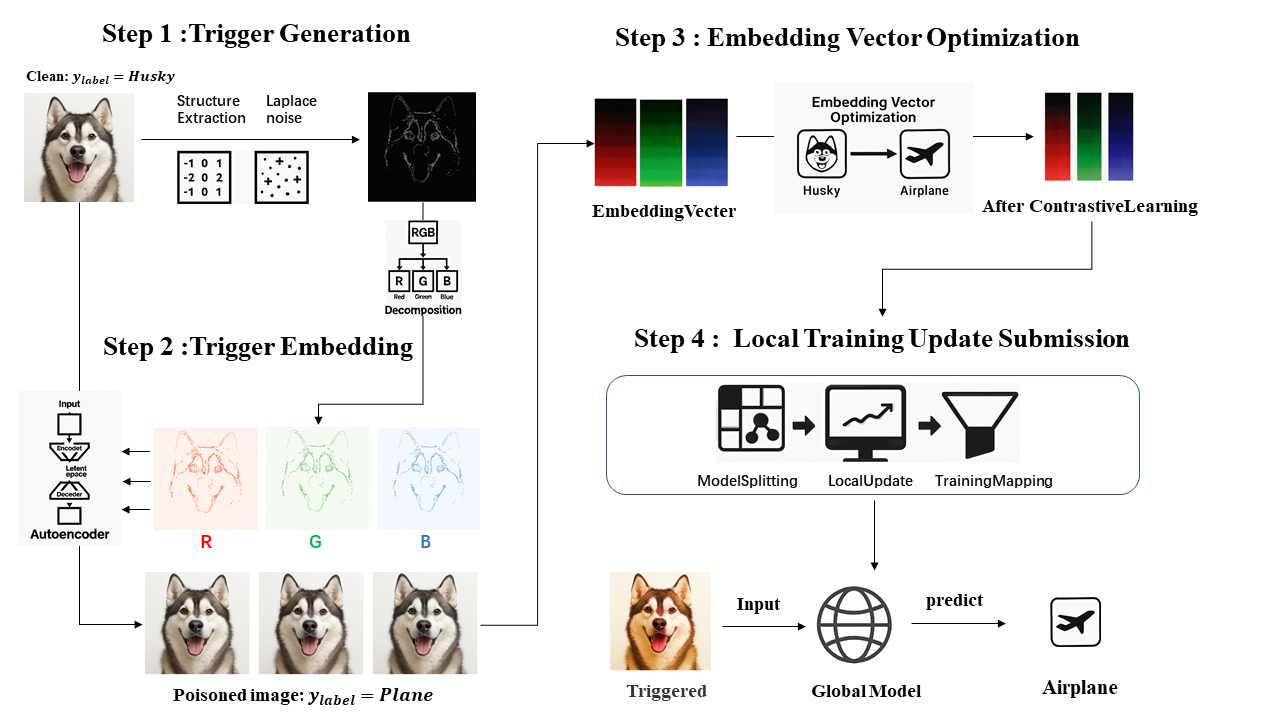}
	\caption{Framework of the Fine-grained Distributed Backdoor Attack.}
	\label{FIG:Framework of the Few-shot Distributed Backdoor Attack.}
\end{figure*}

\begin{enumerate}[1.]
	\item \textbf{Fine-grained Trigger Generation}: Edge features are precisely extracted based on natural image boundaries using adaptive thresholding, and Laplacian noise is strategically injected to generate highly targeted triggers with minimal footprint. The triggers are decomposed into the RGB three channels and distributed across different malicious clients to perform a distributed poisoning attack. This step enhances the stealthiness of the attack and its ability to dynamically adapt.
	\item \textbf{Trigger Embedding}: An autoencoder is used to embed sub-triggers into benign samples to generate poisoned samples. The embedding ensures that the generated poisoned samples visually resemble the original benign samples.
	\item  \textbf{Embedding Vector Optimization}: Contrastive learning is used to optimize the embedding vector of the poisoned samples, bringing it closer to the target class. This amplifies the interference caused by the poisoned samples to the model, achieving a reduction in the amount of poisoning required.
	\item \textbf{Distributed Backdoor Injection}: By separating the model structure and update mapping, we ensure that the updates from the malicious clients remain compatible with the global model.
\end{enumerate}

\subsection{Trigger Generation}
Trigger generation consists of three steps: Edge Structure Extraction, Global Trigger Construction, and Distributed Trigger Decomposition.

\subsubsection{Edge Structure Extraction}
In fine-grained distributed backdoor attacks, precisely targeting edge structures for malicious information injection offers significant advantages in terms of stealth and effectiveness. First, edge structures precisely localize critical regions within the data, enabling more efficient poisoning. Unlike traditional global perturbation methods, edge-based techniques focus on the boundaries of images or data to achieve accurate and high-efficiency poison embedding. Second, the extracted edge structures, after optimization, tend to be sparse, meaning that the embedded poisoned information does not substantially alter the overall data distribution. This greatly reduces the likelihood of detection by anomaly detection mechanisms and improves the stealthiness of the attack.

Edge extraction and optimization serve as the first step in our trigger construction. Malicious clients extract salient edge features from input data to construct dynamic triggers, which are then embedded into selected regions. The effect of this design on attack effectiveness and stealth is evaluated empirically in Section~\ref{subsec:poisoning_efficiency}.

In each round of federated learning, the central server distributes the global model to clients. Upon receiving the global model, a malicious client uses its local dataset $D_{i}=\{(x_{j}, y_{j})\}_{j=1}^{|D_{i}|}$ to generate poisoned samples. For a sample $x_{j} \in D_{i}$ with label $y_{j}$, the attacker predefines a target label $y^{*}$ to be triggered. The goal is to produce an edge map $T'(x_{j})$ as the region for poisoning, enabling covert embedding of triggers.

To extract edge structures, the malicious client performs the following steps:

First, the Canny edge detection algorithm is used to extract an edge map $T(x_{j})$ from each local sample. Then, Gaussian filtering is applied to smooth the image and suppress noise while preserving the dominant edges. The Gaussian kernel function is defined as:

\begin{equation}
	G(i,j)=\frac{1}{2\pi\sigma^{2}}e^{-\frac{i^{2}+j^{2}}{2\sigma^{2}}},
\end{equation}

where $\sigma$ is the smoothing parameter. Malicious clients can tune $\sigma$ to maintain edge clarity despite noise.

Next, the client calculates the gradient of the smoothed image to extract edge intensity and direction. Gradients in the horizontal and vertical directions are computed as:
\begin{equation}\label{eq:xG}
	G_{x}(x,y) = I_{\text{smooth}}(x+1,y) - I_{\text{smooth}}(x-1,y)
\end{equation}

\begin{equation}\label{eq:yG}
	G_{y}(x,y) = I_{\text{smooth}}(x,y+1) - I_{\text{smooth}}(x,y-1)
\end{equation}

Then, edge magnitude is computed as:

\begin{equation}\label{eq:M}
	M(x,y) = \sqrt{G_{x}^{2}(x,y) + G_{y}^{2}(x,y)}
\end{equation}

and edge direction as:

\begin{equation}\label{eq:theta}
	\theta(x,y) = \arctan\left(\frac{G_{y}(x,y)}{G_{x}(x,y)}\right)
\end{equation}

Following this, non-maximum suppression is applied to preserve only the local maxima along edge directions and suppress redundant edge responses. For each pixel, if its gradient magnitude $M(x,y)$ is greater than that of its neighboring pixels in the gradient direction $M_{\text{neighbor}}(x,y)$, it is retained; otherwise, it is suppressed:

\begin{equation}\label{eq:MGxy}
	M'(x,y) = 
	\begin{cases}
		M(x,y), & \text{if } M(x,y) > M_{\text{neighbor}}(x,y) \\
		0, & \text{otherwise}.
	\end{cases}
\end{equation}

Finally, double thresholding is used to classify strong and weak edges. With a high threshold $T_{H}$ and a low threshold $T_{L}$, a binary edge map $T(x_{j})$ is generated as:

\begin{equation}\label{eq:GT}
	T(x_{j})=
	\begin{cases}
		1, & \text{if } M'(x,y) > T_{H} \\
		0, & \text{if } M'(x,y) < T_{L}.
	\end{cases}
\end{equation}

This strategy prioritizes stronger local edge responses when constructing the trigger mask; its contribution is evaluated through the edge-extraction ablation in Table~\ref{tab:ablation}.

\subsubsection{Global Trigger Generation}

This step generates a global trigger by injecting perturbations into the edge structure $T(x_{j})$. The combination of edge maps and Laplacian noise provides localized, stochastic control over the trigger perturbation. Its effects on attack performance and visual quality are evaluated empirically rather than attributed to prior work.

To achieve this, we sample perturbations from a Laplace distribution, whose use as a calibrated noise distribution is well established \cite{dwork2006calibrating}. In FDBA, the parameters $\mu$ and $b$ control the location and scale of the stochastic trigger perturbation; the resulting stealth--effectiveness trade-off is assessed experimentally.

The poisoning information $t$ is generated according to the Laplace distribution, whose probability density function (PDF) is defined as:

\begin{equation}\label{eq:PDF}
	f(t|\mu,b) = \frac{1}{2b}\exp\left(-\frac{|t - \mu|}{b}\right)
\end{equation}

where $\mu$ is the location parameter, controlling the center of the noise, and $b$ is the scale parameter, which determines the spread. In our setup, we fix $\mu = 0$ to ensure symmetric noise and adjust $b$ empirically to balance stealthiness and effectiveness.

By integrating Laplacian noise into the edge structure, we can generate a dynamic and covert trigger capable of executing effective backdoor attacks in a federated learning setting. Specifically, for a given sample $x$, the poisoned information $t(x)$ is embedded into its extracted edge map $T'(x)$ as follows:

\begin{equation}\label{eq:t(x)}
	t(x) = \alpha \cdot T'(x) \cdot r(x)
\end{equation}

where:
$\alpha$ is the strength control parameter, regulating the impact of the noise on the trigger;
$T'(x)$ is the sparse edge mask generated by the edge extraction module, guiding precise injection locations;
$r(x) \sim \text{Laplace}(\mu, b)$ is Laplacian noise used to increase stealth.

By adjusting these parameters ($\alpha$,$T'(x)$,$r(x)$), attackers can adapt the trigger to different scenarios while ensuring high stealth and effectiveness. In federated learning environments, where client data distributions may vary, it is necessary to dynamically adapt the poisoning strength. The introduction of the strength control parameter $\alpha$ allows flexible tuning of the trigger's influence on the global model, thus reducing the risk of detection while maintaining the attack's impact.

To ensure dimensional compatibility between $T'(x)$ and $r(x)$, we sample $r(x)$ with the same spatial dimensions as $T'(x)$. By tuning $\alpha$, attackers can strike a balance between stealthiness and poisoning effectiveness in distributed learning settings.

\subsubsection{Distributed Trigger Decomposition}

We propose a multi-channel global trigger decomposition strategy that separates the trigger across the RGB channels, allowing different malicious clients to embed different components of the trigger in a distributed manner while preserving the overall effect. The core idea behind this RGB-based decomposition strategy lies in leveraging the collaborative nature of distributed attacks. By decomposing the global trigger into its R, G, and B color components and distributing them across different malicious clients, the trigger is embedded in a distributed manner while preserving the natural color statistics of the host images. This channel-wise decomposition prevents any single client from holding the complete trigger, enhancing stealthiness and hindering anomaly detectors that rely on identifying coordinated, monolithic trigger patterns.

The RGB decomposition of the poisoned edge feature map $T'(x)$ is defined as:

\begin{equation}\label{eq:T(x)RGB}
	T'(x) = T'_R(x) + T'_G(x) + T'_B(x)
\end{equation}

where:
$T'_R(x)$, $T'_G(x)$, and $T'_B(x)$ represent the sub-triggers distributed over the R, G, and B channels, respectively;
Each channel-specific feature map is generated via weighted assignment and perturbation mechanisms.

To ensure the effective fusion of the three sub-triggers into a coherent attack, a channel-wise weighting strategy is employed to modulate the contribution of each edge map:

\begin{equation}\label{eq:TCx}
	T'_C(x) = w_C \cdot T'(x)
\end{equation}

where $C \in \{R, G, B\}$ denotes the current channel, and $w_C$ is the weight assigned to that channel. The weights are determined through grid search and constrained such that:

\begin{equation}\label{eq:weights_sum}
	\sum_{C \in \{R,G,B\}} w_C = 1
\end{equation}

This decomposition enables distributed poisoning across multiple clients while preserving the structure of the original trigger, since the three channel components jointly reconstruct it. RGB decomposition is a color-channel operation and is not a spatial-frequency decomposition. Its benefit is not improved image-level imperceptibility but reduced similarity among the updates uploaded by different malicious clients, which is precisely the signal that similarity-based detectors such as FoolsGold rely on (Table~\ref{tab:detection_metrics}).

\subsection{Trigger Embedding}
In the distributed backdoor attack scenario, the global trigger $T(x)$ has been decomposed into multiple sub-triggers $T'_R(x), T'_G(x), T'_B(x)$, which are assigned to different malicious clients. The decomposition operates over the color channels of the trigger: the trigger is split into its R, G and B components, and no spatial-frequency interpretation is implied. By decomposing the trigger into multiple sub-components, each can be independently embedded into the client's data while preserving the overall functionality of the trigger. This approach has natural advantages in a distributed environment, as each client only processes part of the trigger information, making it difficult for the server to directly detect the distributed trigger, thereby enhancing the stealthiness of the attack.

The goal of this section is to embed sub-triggers into benign samples using an autoencoder to generate poisoned samples. Assume there are $N$ malicious clients, and each client $i$ is assigned a sub-trigger $T'_C(x)$ corresponding to channel $C \in \{R, G, B\}$. Client $i$ embeds the assigned sub-trigger into its local benign sample to generate a poisoned sample:

\begin{equation}\label{eq:xij_prime}
	x'_{i,j} = x_{j} + \lambda_{C} \cdot T'_{C}(x_{j})
\end{equation}

where $x_{j}$ is the original benign sample from the client's local dataset; $\lambda_{C}$ is a parameter controlling the strength of the sub-trigger, used to adjust the degree of influence the sub-trigger has on the sample. By decomposing and distributing the global trigger among multiple clients, the stealthiness of the attack is improved, as each client holds only a portion of the trigger information, making it harder to detect centrally.

Inspired by \cite{feng2019learning}, this paper adopts an autoencoder for trigger embedding. The autoencoder consists of two parts: an encoder $E$ and a decoder $D$. By learning a low-dimensional representation of the data, the autoencoder can embed sub-triggers while preserving the main features of the data. Specifically: the encoder maps the input sample into a latent variable space to capture the key features of the data; the decoder reconstructs the sample from the latent variables, ensuring that the generated poisoned sample is visually consistent with the original sample. Embedding the sub-trigger in the latent space enables stealthy embedding without significantly altering the data distribution. The design is as follows:

\begin{enumerate}
	\item  Encoder: maps the input sample $x$ to the latent space: $z = E(x), \quad z \in \mathbb{R}^{d}$.
	\item  Decoder: reconstructs the latent variable $z$ back to the original sample: $\hat{x} = D(z)$.
\end{enumerate}

The sub-trigger is embedded in the latent space of the encoder-generated latent variable $z$, resulting in a perturbed latent variable:$z' = z + \lambda \cdot T'_{C}(x),$

where $\lambda$ is the embedding strength control parameter. The perturbed latent variable $z'$ is decoded into a poisoned sample: $x' = D(z')$

By jointly optimizing the reconstruction loss and trigger embedding loss, a balance between stealthiness and attack effectiveness can be achieved. The reconstruction loss ensures the stealthiness of the poisoned samples, making them difficult to detect by conventional anomaly detection mechanisms. The trigger embedding loss ensures effective embedding of the sub-trigger in the latent space, thereby enabling the backdoor behavior under specific inputs. The optimization objective of the autoencoder includes the following two loss functions:

\begin{itemize}
	\item Reconstruction Loss: ensures that the poisoned sample $x'$ is visually close to the original sample $x$, to maintain stealthiness. It is defined as: $\mathcal{L}_{\text{recon}} = \|x - x'\|^2$.
	\item Trigger Embedding Loss: ensures that the sub-trigger $T'_{C}(x)$ is effectively embedded in the latent space while preserving the attack trigger effect. The goal is to maximize the similarity between the latent representation of the poisoned sample and the target trigger embedding: $\mathcal{L}_{\text{trigger}} = \|E(T'_{C}(x')) - E(T'_{C}(x_{p}))\|^2,$.
	where $x_{p}$ is a sample from the target class (promoter). The final optimization objective is:
	\begin{equation}
		\mathcal{L}_{\text{total}} = \alpha \cdot \mathcal{L}_{\text{recon}} + \beta \cdot \mathcal{L}_{\text{trigger}},
	\end{equation}
	where $\alpha$ and $\beta$ are weight parameters used to balance stealthiness and trigger embedding.
\end{itemize}

Combining the above sub-trigger embedding and autoencoder optimization, the poisoned sample generation formula is:
\begin{equation}
	x'_{j} = D\left(E(x_{j}) + \lambda \cdot T'_{C}(x_{j})\right),
\end{equation}
where $E(x_{j})$ and $D$ are the encoder and decoder of the autoencoder; $T'_{C}(x_{j})$ is the sub-trigger assigned to the client; $\lambda$ is the embedding strength parameter.

\subsection{Embedding Vector Optimization}

Embedding vector optimization of poisoned samples is a key step in enabling low-poisoning sample attacks. It includes two main stages: contrastive learning and random projection. Contrastive learning pulls the embedding vectors of poisoned samples closer to those of the target class while pushing them away from other classes. This operation enables benign samples of the target class to amplify the effect of poisoned samples, reducing the overall poisoning rate required for the attack. Furthermore, to address the high computational complexity of contrastive learning, we adopt random projection to map high-dimensional embedding vectors into a low-dimensional space, simplifying the model’s handling complexity for poisoned data. This optimization strategy effectively amplifies the impact of poisoned samples on the global model, maintaining attack efficacy even with a limited number of poisoned instances. The full process consists of embedding vector extraction, optimization, and random projection.

\subsubsection{Embedding Vector Extraction}

This paper adopts the feature extraction module of ResNet-18 as the base network for embedding vector extraction, removing the fully connected classification layer and using the output from the global average pooling (GAP) layer before the final FC layer to produce high-dimensional feature representations of input samples. The embedding extraction architecture includes:

\begin{enumerate}
	\item Input Layer: Input images of size $224 \times 224 \times 3$, processed by a $7 \times 7$ convolution followed by a $2 \times 2$ max-pooling layer to extract low-level features.
	\item Residual Modules: Four stages of residual blocks, each with several convolution layers and skip connections:
	\begin{itemize}
		\item Stage 1: 64 filters of size $3 \times 3$, two residual blocks
		\item Stage 2: 128 filters, two residual blocks
		\item Stage 3: 256 filters, two residual blocks
		\item Stage 4: 512 filters, two residual blocks
	\end{itemize}
	\item GAP Layer: The final output feature map is compressed to a 512-dimensional vector via global average pooling: $e = \text{GAP}(f_{\text{conv4}}(x)),$, where $f_{\text{conv4}}(x)$ denotes the output of stage 4.
	\item Embedding Vector Output: The 512-dimensional output $e$ represents the high-dimensional feature of sample $x$. To better adapt to the task of distributed backdoor attacks, the ResNet-18 backbone is enhanced with a channel attention mechanism in each residual block and an additional embedding head (e.g., fully connected or MLP) after the GAP layer: $e = \text{MLP}(\text{GAP}(f_{\text{conv4}}(x))).$
\end{enumerate}

\subsubsection{Embedding Vector Optimization}

Embedding vector optimization aims to enhance specific distributional characteristics by adjusting pairwise sample distances. The core idea is to enforce similar samples to have similar representations and dissimilar samples to have different ones. Hadsell et al. introduced a contrastive loss that pulls similar pairs together and pushes dissimilar pairs apart in the learned representation space \cite{hadsell2006dimensionality}.

This module applies a contrastive learning objective to optimize the embedding distribution of poisoned samples $e'_i$ with respect to promoters $e_p$ and distractors $e_n$. The objectives include:

\begin{enumerate}
	\item \textbf{Attraction to Target Class}: Minimize the distance between poisoned sample embeddings and the promoter class:
	\begin{equation}
		\mathcal{L}_{\text{positive}} = \frac{1}{N} \sum_{i=1}^{N} \|e'_i - e_p\|^2,
	\end{equation}
	where $N$ is the number of poisoned samples.
	\item \textbf{Repulsion from Non-Target Classes}: Maximize the distance from distractor embeddings:
	\begin{equation}
		\mathcal{L}_{\text{negative}} = -\frac{1}{N} \sum_{i=1}^{N} \|e'_i - e_n\|^2.
	\end{equation}
	\item \textbf{Combined Contrastive Objective}: Incorporate a margin $m$ to control the separation between classes,The contrastive learning objective in Eq.~(18) serves a dual purpose: it not only optimizes the embedding distribution but also provides the gradient amplification factor $\gamma$ crucial for our theoretical analysis (see Theorem~\ref{thm:main}).
	\begin{equation}
		\mathcal{L}_{\text{contrastive}} = \frac{1}{N} \sum_{i=1}^{N} \left[\|e'_i - e_p\|^2 - \|e'_i - e_n\|^2 + m\right]_+,
	\end{equation}
	where $[\cdot]_+$ is the ReLU function that enforces non-negativity by setting negative values to zero.
\end{enumerate}

\subsubsection{Random Projection}
Since embedding vector optimization often occurs in high-dimensional space (e.g., 512D), computing distances for all pairs is computationally expensive. In federated settings, this becomes even more demanding. As shown in $\mathcal{L}_{\text{positive}}$ and $\mathcal{L}_{\text{negative}}$, each operation is $O(d)$, leading to a total complexity of $O(Nd)$.

To mitigate this, we introduce Random Projection Hashing (RPH) \cite{shi2009hash}, which projects high-dimensional vectors into lower-dimensional space:

\begin{enumerate}
	\item \textbf{Random Gaussian Matrix}: Generate $R \in \mathbb{R}^{k \times d}$, with $R_{ij} \sim \mathcal{N}(0,1)$, where $k \ll d$.
	\item \textbf{Low-Dimensional Projection}: Project high-dimensional vector $e$ to $h(e) \in \mathbb{R}^k$: $h(e) = \text{sign}(R \cdot e),$, where $\text{sign}(\cdot)$ binarizes the result, reducing both storage and compute cost. The contrastive loss becomes:
	\begin{equation}
		\begin{aligned}
			\mathcal{L}_{\text{contrastive-hash}} = \frac{1}{N} \sum_{i=1}^{N} \Bigl[ & \|h(e'_i) - h(e_p)\|^2 \\
			& - \|h(e'_i) - h(e_n)\|^2 + m \Bigr]_+.
		\end{aligned}
	\end{equation}
	The Johnson-Lindenstrauss Lemma \cite{johnson1984extensions} ensures that distances are approximately preserved in the projected space. For any set $\{e_1, ..., e_N\} \subset \mathbb{R}^d$, there exists $R \in \mathbb{R}^{k \times d}$ such that:
	\begin{equation}
		(1 - \epsilon) \|e_i - e_j\|^2 \leq \|h(e_i) - h(e_j)\|^2 \leq (1 + \epsilon) \|e_i - e_j\|^2,
	\end{equation}
	with $\epsilon$ typically set to a small value (e.g., 0.1), and $k = O(\frac{\log N}{\epsilon^2})$. The original complexity is $O(Nd)$; after projection, the complexity becomes $O(Ndk + Nk)$, and since $k \ll d$, it approximates $O(Nk)$, significantly reducing computation.
\end{enumerate}
In federated learning, optimizing poisoned sample embeddings is critical to improving attack effectiveness. Poisoned embeddings are pulled closer to target class semantics while staying distant from others. Random projection significantly reduces computational costs while preserving semantic relationships, thereby enhancing attack robustness and efficiency—even with only a few poisoned samples.

\subsection{Distributed Backdoor Injection}

As malicious clients train using optimized embedding vectors, discrepancies arise in model update structures compared to benign clients. To ensure structural consistency, the attack process is divided into three phases: model structure decoupling, local training, and update mapping. Malicious clients split their local models into an additional structure (for processing embedding vectors) and a globally consistent part (identical to the global model architecture). A mapping strategy is used to ensure the uploaded updates are statistically close to benign ones, thereby evading anomaly detection mechanisms.

\subsubsection{Model Structure Decoupling}

The local model is divided into two functional modules. The additional structure is responsible for converting the embedding vector $e'_i \in \mathbb{R}^d$ into an input format compatible with the global model (e.g., image feature $H \times W \times C$). This is achieved via two operations:

First, a fully connected layer projects the embedding vector to a lower-dimensional representation:

\begin{equation}
	h_{\text{adapted}} = \text{ReLU}(W_{\text{fc}} \cdot e'_i + b_{\text{fc}}),
\end{equation}

where $W_{\text{fc}} \in \mathbb{R}^{d' \times d}$ and $b_{\text{fc}} \in \mathbb{R}^{d'}$ are the weights and bias, and $d'$ is the reduced dimension.

Second, a transposed convolution layer reconstructs the reduced features into a fake input:

\begin{equation}
	x'_{\text{fake}} = \text{ConvTranspose}(h_{\text{adapted}}),
\end{equation}

whose shape matches that of the global model input $x_i$.

The globally consistent part has an architecture strictly identical to the global model on the server and processes $x'_{\text{fake}}$ to produce the prediction $f_{\text{global}}(x'_{\text{fake}})$.

\subsection{Local Training}

The training process consists of two phases:
\begin{enumerate}
	\item \textbf{Input Generation}: The additional structure converts the embedding vector into a fake input: $x'_{\text{fake}} = f_{\text{adapt}}(e'_i),$, where $f_{\text{adapt}}$ represents the entire operation of the additional structure.
	\item  \textbf{Model Optimization}: The fake input $x'_{\text{fake}}$ is passed through the globally consistent part $f_{\text{global}}$ and optimized via an attack loss function: 
	\begin{equation}
		\mathcal{L}_{\text{attack}} = \frac{1}{|D'|} \sum_{i=1}^{|D'|} \ell(f_{\text{global}}(x'_{\text{fake}}), y^*),
	\end{equation}
	where $\ell(\cdot)$ is the loss function (e.g., cross-entropy), $y^*$ is the target backdoor label, and $D'$ is the local dataset of the malicious client. This process forces the model to produce predictions aligned with the backdoor target label.
\end{enumerate}

\subsubsection{Update Mapping}

After training, the malicious client generates the final update in three steps:

\begin{enumerate}
	\item \textbf{Update Separation}: Compute updates for the additional structure and globally consistent part:
	\begin{equation}
		\Delta w_{\text{adapt}} = w_{\text{adapt}}^{t+1} - w_{\text{adapt}}^t,\quad 
		\Delta w_{\text{global}} = w_{\text{global}}^{t+1} - w_{\text{global}}^t.
	\end{equation}
	\item \textbf{Update Alignment}: Only retain $\Delta w_{\text{global}}$ and align it with a simulated benign update $\Delta w_{\text{benign}}$ using a mixing ratio $\beta \in [0,1]$:
	\begin{equation}
		\Delta w_{\text{aligned}} = \beta \cdot \Delta w_{\text{benign}} + (1 - \beta) \cdot \Delta w_{\text{global}}.
	\end{equation}
	The simulated benign update is computed as the average gradient over benign data:
	\begin{equation}
		\Delta w_{\text{benign}} = \frac{1}{|D|} \sum_{i=1}^{|D|} \nabla_w \ell(f_{\text{global}}(x_i), y_i),
	\end{equation}
	
	where $x_i$ and $y_i$ are benign inputs and corresponding labels.
	\item \textbf{Final Upload}: The malicious client uploads:
	\begin{equation}
		\Delta w_{\text{final}} = \Delta w_{\text{aligned}},
	\end{equation}
	ensuring the update conforms to the global model structure and closely mimics the distribution of benign updates.
	
\end{enumerate}

Through model structure decoupling, the malicious client’s local model includes additional modules but only uploads updates from the globally consistent part, preventing incompatibility issues. Update alignment ensures statistical similarity with benign updates, significantly reducing the risk of detection by anomaly detection algorithms. This strategy ensures the feasibility of federated aggregation while successfully injecting backdoors.

\subsection{Summary and Algorithm Description of FDBA}
This chapter provides a comprehensive overview of the proposed method. As illustrated in Algorithm \ref{alg:attack}, the approach first employs a multimodal trigger generation mechanism that combines Canny edge detection with Laplace noise to produce dynamic triggers, which are then injected in a distributed manner using an RGB three-channel decomposition strategy. Next, an adaptive embedding method based on an autoencoder is designed to embed triggers effectively while maintaining visual stealth. Contrastive learning and random projection techniques are introduced to optimize the distribution of poisoned samples in the embedding space, enhancing attack effectiveness while reducing computational complexity. During execution, a model structure decoupling and update mapping mechanism is adopted. By separating additional network structures and aligning parameter update distributions, the method effectively evades anomaly detection in federated learning. Through the coordinated operation of four modules—dynamic trigger generation, distributed injection, embedding space optimization, and parameter update obfuscation—this framework ensures high attack success rates while significantly reducing the required amount of poisoned data.

\begin{algorithm}[t]
	\caption{Fine-grained Distributed Backdoor Attack in Federated Learning}
	\label{alg:attack}
	\begin{algorithmic}[1]
		\Procedure{Trigger Generation}{$x_j$}
		\State Extract edge map $T(x_j)$ via Canny detector
		\State Apply Gaussian filter: $I_{\text{smooth}} \gets G \ast x_j$
		\State Compute gradients:
		\State \quad $G_x \gets I_{\text{smooth}}(x+1,y) - I_{\text{smooth}}(x-1,y)$
		\State \quad $G_y \gets I_{\text{smooth}}(x,y+1) - I_{\text{smooth}}(x,y-1)$
		\State \quad $M(x,y) \gets \sqrt{G_x^2 + G_y^2}$
		\State Non-maximum suppression: $M^{\prime}(x,y) \gets \text{Threshold}(M)$
		\State Generate edge mask: $T(x_j) \gets \text{DualThresholding}(M^{\prime})$
		\State Inject noise:
		\State \quad $T^{\prime}(x_j) \gets T(x_j) + \alpha \cdot \text{Laplace}(0,b)$
		\State Channel decomposition:
		\State \quad $T^{\prime}_R, T^{\prime}_G, T^{\prime}_B \gets T^{\prime}(x_j)$
		\EndProcedure
		
		\Procedure{Trigger Embedding}{$x_j, T^{\prime}_C$}
		\State Feature encoding: $z \gets E(x_j)$
		\State Latent space perturbation: $z^{\prime} \gets z + \lambda \cdot T^{\prime}_C$
		\State Sample generation: $x^{\prime}_j \gets D(z^{\prime})$
		\State Loss optimization:
		\State \quad $\mathcal{L}_{\text{total}} \gets \alpha\mathcal{L}_{\text{recon}} + \beta\mathcal{L}_{\text{trigger}}$
		\EndProcedure
		
		\Procedure{Embedding Optimization}{$x^{\prime}_j$}
		\State Feature extraction: $e^{\prime}_j \gets \text{ResNet-18}(x^{\prime}_j)$
		\State Contrastive learning:
		\State \quad $\mathcal{L}_{\text{contrast}} \leftarrow \sum_{i=1}^N[\|e'_i - e_p\|^2 - \|e'_i - e_n\|^2 + m]_+$ 
		// Amplification factor $\gamma$
		\State Feature binarization: $h(e^{\prime}_j) \gets \text{sign}(R \cdot e^{\prime}_j)$
		\EndProcedure
		
		\Procedure{Distributed Injection}{$D^{\prime}, y^*$}
		\State Model separation:
		\State \quad $f_{\text{adapt}}, f_{\text{global}} \gets \text{SplitModel}(D^{\prime})$
		\State Fake input generation: $x^{\prime}_{\text{fake}} \gets f_{\text{adapt}}(e^{\prime}_j)$
		\State Attack loss computation:
		\State \quad $\mathcal{L}_{\text{attack}} \gets \ell(f_{\text{global}}(x^{\prime}_{\text{fake}}), y^*)$
		\State Update alignment:
		\State \quad $\Delta w_{\text{aligned}} \gets \beta\Delta w_{\text{benign}} + (1-\beta)\Delta w_{\text{global}}$
		\State Parameter upload: Upload $\Delta w_{\text{aligned}}$ to server
		\EndProcedure
	\end{algorithmic}
\end{algorithm}

\section{Theoretical Analysis}

This section presents an analytical characterization of FDBA's poisoning efficiency relative to DBA under the modeling assumptions stated in Appendix~\ref{sec:Proof}.

\subsection{Problem Formulation}

Consider a federated learning system with $K$ clients, where $K_m \subset K$ are malicious clients. Let $p \in [0,1]$ denote the poisoning ratio (fraction of poisoned samples), and let $\text{ASR}(p)$ denote the attack success rate as a function of $p$.

\begin{definition}[Attack Success Rate]
	The attack success rate is defined as:
	\begin{equation}
		\text{ASR}(p) = \mathbb{P}[f_\theta(T(x)) = y^* | x \sim \mathcal{D}_\text{test}]
	\end{equation}
	where $f_\theta$ is the global model, $T(\cdot)$ is the trigger function, and $y^*$ is the target label.
\end{definition}

\subsection{Key Result}

Our main theoretical contribution is formalized in the following theorem:

\begin{theorem}[Model-Based Poisoning Efficiency of FDBA]
	\label{thm:main}
	Under the convexity and bounded-gradient assumptions and the cumulative-influence model specified in Appendix~\ref{sec:Proof}, the poisoning-ratio relation at equal modeled attack success rate is:
	\begin{equation}
		\rho = \frac{p_{\text{FDBA}}}{p_{\text{DBA}}} = \frac{1}{1 + \gamma \cdot \alpha}
	\end{equation}
	where $\gamma > 0$ is the contrastive learning gain and $\alpha \in (0,1]$ is the edge structure preservation factor.
	And $\rho$ represents the poisoning ratio reduction factor, indicating that FDBA requires only $\rho$ fraction of the poisoning samples needed by DBA.
\end{theorem}

The proof relies on two key mechanisms:
\begin{enumerate}
	\item \textbf{Gradient Amplification}: Contrastive learning amplifies the gradient contribution of poisoned samples by factor $(1 + \gamma)$.
	\item \textbf{Structural Coherence}: Edge-based triggers maintain consistency across heterogeneous data distributions with preservation factor $\alpha$.
\end{enumerate}

\subsection{Practical Implications}

For illustrative parameter values ($\gamma = 0.3$, $\alpha = 0.8$), we obtain:
\begin{equation*}
	\rho = \frac{1}{1 + 0.3 \times 0.8} \approx 0.8065,
\end{equation*}
corresponding to an illustrative model-based reduction of approximately 19.4\% in poisoning ratio.

The empirical estimates range from 37.4\% to 48.4\% (Section~\ref{subsec:poisoning_efficiency}). Because $\gamma$ and $\alpha$ are not independently estimated from the experiments, this comparison is interpreted as qualitative consistency with the analytical relation rather than numerical validation of a universal lower bound. The derivation is provided in Appendix~\ref{sec:Proof}.

\section{Experiments}

Evaluating fine-grained backdoor attacks in federated learning requires careful consideration of diverse scenarios, defense mechanisms, and the precision of attack targeting strategies. This section presents comprehensive experiments designed to validate FDBA's effectiveness while acknowledging its limitations. Our evaluation spans multiple datasets, federated configurations, and defense strategies to provide a thorough understanding of the proposed framework's capabilities and boundaries.

\subsection{Experimental Configuration}

Our experimental setup mirrors realistic federated learning deployments while maintaining scientific rigor. We conduct experiments across three benchmark datasets: CIFAR-10 and CIFAR-100 for controlled evaluation, and a downsampled ImageNet-32 variant for high-complexity scenarios. The choice of these datasets reflects the progression from simple classification tasks to more challenging real-world image recognition problems commonly encountered in federated applications \cite{li2020federated}.

The federated learning environment simulates 100 participating clients, following the scale commonly adopted in federated learning research \cite{kairouz2021advances}. In each communication round, we randomly select 10\% of clients for model updates, reflecting the partial participation typical in practical federated systems. Data heterogeneity is modeled using the Dirichlet distribution with concentration parameter $\alpha \in \{0.1, 0.5, 1.0, 2.0\}$, where smaller values indicate higher non-IID characteristics \cite{hsu2019measuring}. This approach allows us to systematically evaluate attack performance across varying degrees of data heterogeneity.

The attack configuration considers malicious client ratios ranging from 5\% to 20\%, with individual poisoning ratios varying between 1\% and 10\%. These parameters reflect realistic threat scenarios where attackers have limited control over the federated network. Our trigger generation employs Canny edge detection with thresholds $T_H = 0.25$ and $T_L = 0.15$, chosen to balance edge sensitivity with noise robustness. The Laplacian noise injection uses parameters $\mu = 0$ and $b = 0.35$ for CIFAR datasets, with $b = 0.45$ for ImageNet-32 to accommodate higher complexity.

Our primary comparisons focus on DBA, a distributed trigger-decomposition attack, and Neurotoxin, which targets parameters that change less during benign training to improve backdoor durability \cite{zhang2022neurotoxin}.

Defense mechanisms span both traditional and advanced approaches. Byzantine-robust aggregation methods include Krum \cite{blanchard2017machine} and Multi-Krum, which select updates based on geometric distance measures. We implement FoolsGold \cite{fungMitigatingSybilsFederated2020}, which exploits the assumption that malicious clients produce similar gradient updates. FLAME \cite{nguyen2022flame} clusters and clips suspicious updates before aggregation. Additionally, we evaluate FedRecover \cite{cao2023fedrecover}, a post-detection recovery method that uses historical training information to reconstruct an accurate global model after poisoning, and differential privacy mechanisms with noise parameters $\epsilon \in \{0.1, 1.0, 10.0\}$ \cite{abadi2016deep}.

\subsection{Evaluation Metrics and Statistical Framework}

Our evaluation employs multiple complementary metrics to capture different aspects of attack effectiveness and stealth. The primary metric, Attack Success Rate (ASR), measures the probability that trigger-embedded samples are misclassified into the target class:

\begin{equation}
	\text{ASR} = \frac{|\{x_i : f(T(x_i)) = y^*, x_i \in \mathcal{D}_{\text{test}}\}|}{|\{x_i : T(x_i), x_i \in \mathcal{D}_{\text{test}}\}|} \times 100\%
\end{equation}

where $T(\cdot)$ represents the trigger function, $y^*$ is the target label, and $\mathcal{D}_{\text{test}}$ is the test dataset. Clean Accuracy (CA) ensures that backdoor injection does not compromise the model's primary functionality:

\begin{equation}
	\text{CA} = \frac{|\{x_i : f(x_i) = y_i, (x_i, y_i) \in \mathcal{D}_{\text{clean}}\}|}{|\mathcal{D}_{\text{clean}}|} \times 100\%
\end{equation}

To quantify visual stealth, we compute Peak Signal-to-Noise Ratio (PSNR) and Structural Similarity Index (SSIM):

\begin{equation}
	\text{PSNR} = 10 \log_{10} \left( \frac{\text{MAX}_I^2}{\text{MSE}(I, I')} \right)
\end{equation}

\begin{equation}
	\text{SSIM}(I, I') = \frac{(2\mu_I\mu_{I'}+C_1)(2\sigma_{II'}+C_2)}{(\mu_I^2+\mu_{I'}^2+C_1)(\sigma_I^2+\sigma_{I'}^2+C_2)}
\end{equation}

where $I$ and $I'$ represent original and poisoned images, respectively. Perturbation magnitude is measured using $L_0$ and $L_2$ norms to assess both sparsity and intensity of modifications.

Results are reported as means and standard deviations where repeated measurements are available. We do not report inferential significance tests because the retained result records contain aggregate summaries rather than the per-run observations required to reproduce such tests.

\subsection{Attack Effectiveness and Poisoning Efficiency}
\label{subsec:poisoning_efficiency}

The fundamental question driving our evaluation concerns FDBA's ability to achieve high attack success rates efficiently. To establish the core capability of FDBA, we first evaluate its attack effectiveness against baseline methods under constrained poisoning budgets, without the presence of active defense mechanisms. 

Under a standard baseline configuration of 10\% malicious clients and a 4\% poisoning ratio, FDBA demonstrates consistent superiority. Without defense mechanisms, FDBA achieves 82.4\% ASR, representing a gain of 12.6 percentage points over Neurotoxin (69.8\%) and 15.2 percentage points over DBA (67.2\%). This enhancement stems from our embedding optimization strategy, which amplifies the influence of individual poisoned samples through contrastive learning.

To further demonstrate FDBA's superior efficiency, we evaluate the attack success rates across different poisoning ratios. 

\begin{table}[t]
	\centering
	\caption{Attack Success Rate vs Poisoning Ratio on CIFAR-10 (No Defense)}
	\label{tab:poisoning_ratio}
	\footnotesize
	\begin{tabular}{c|c|c|c|c}
		\hline
		\textbf{Ratio} & \textbf{DBA} & \textbf{Neurotoxin} & \textbf{FDBA} & \textbf{Gain (p.p.)} \\
		\hline
		1\% & 32.1$\pm$6.8 & 35.4$\pm$6.9 & \textbf{51.2$\pm$5.3} & +15.8 \\
		2\% & 48.6$\pm$5.9 & 52.1$\pm$5.7 & \textbf{67.8$\pm$4.6} & +15.7 \\
		3\% & 59.7$\pm$5.2 & 62.8$\pm$5.0 & \textbf{75.9$\pm$4.2} & +13.1 \\
		4\% & 67.2$\pm$4.8 & 69.8$\pm$4.7 & \textbf{82.4$\pm$3.8} & +12.6 \\
		5\% & 74.1$\pm$4.5 & 76.3$\pm$4.4 & \textbf{87.6$\pm$3.5} & +11.3 \\
		6\% & 79.3$\pm$4.2 & 80.9$\pm$4.1 & \textbf{89.6$\pm$3.2} & +8.7 \\
		8\% & 85.7$\pm$3.6 & 86.8$\pm$3.5 & \textbf{93.1$\pm$2.7} & +6.3 \\
		10\% & 90.1$\pm$3.0 & 90.8$\pm$3.1 & \textbf{95.3$\pm$2.4} & +4.5 \\
		\hline
	\end{tabular}
\end{table}

\begin{figure*}[h]
	\centering
	\includegraphics[width=0.8\textwidth]{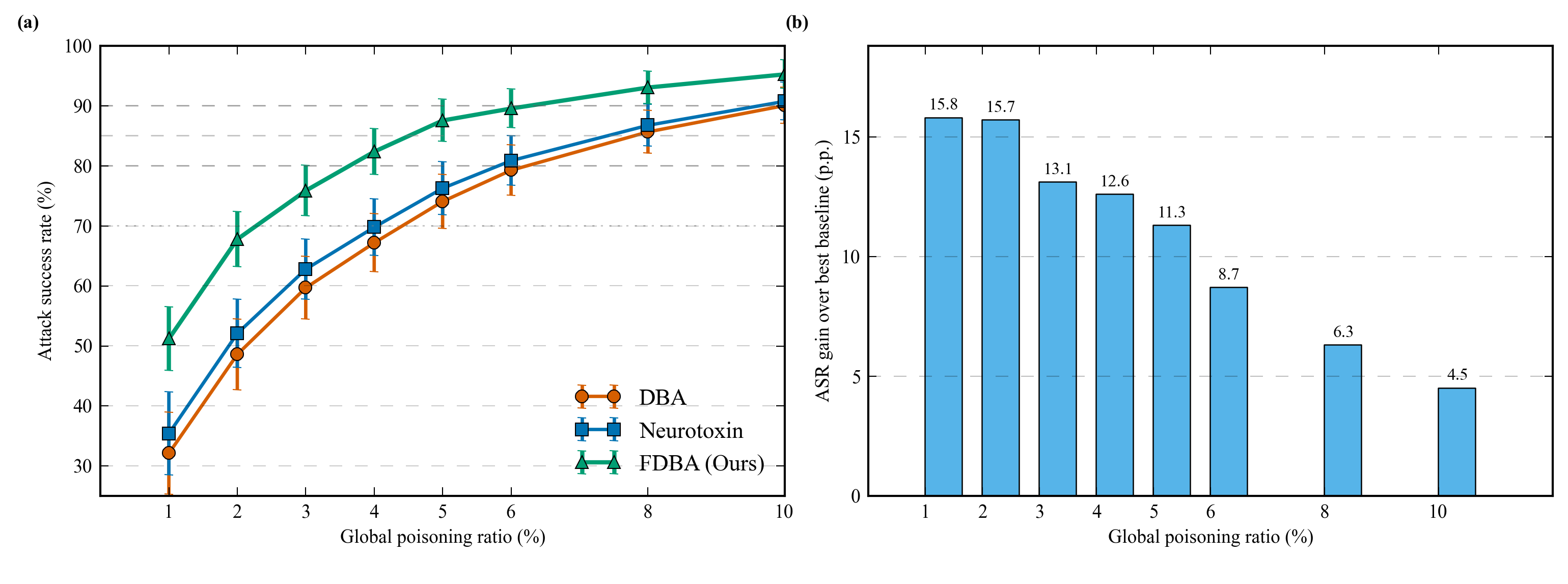}
	\caption{Analysis of poisoning efficiency and attack scalability. (a) Mean ASR with standard-deviation error bars across eight measured poisoning ratios; horizontal reference lines mark the target ASRs used for piecewise-linear interpolation. (b) FDBA's absolute ASR gain over the strongest baseline, reported in percentage points.}
	\label{FIG:fig2_poisoning_efficiency.}
\end{figure*}

Figure \ref{FIG:fig2_poisoning_efficiency.} visualizes the efficiency advantages quantified in Table \ref{tab:poisoning_ratio}. Panel (a) reports the measured means and standard deviations used for the interpolation analysis and shows that FDBA achieves higher ASR than the strongest baseline at every measured poisoning ratio. Panel (b) reports the absolute ASR gain, which decreases from 15.8 to 4.5 percentage points as the poisoning ratio increases. This pattern indicates that the largest measured advantage occurs in the most resource-constrained settings.

While the aforementioned results demonstrate FDBA's superior effectiveness under fixed constrained budgets, we introduce a \textbf{Poisoning Efficiency Curve analysis} to estimate how many poisoning samples FDBA saves relative to baseline methods at the same target ASR.

To this end, we measured ASR at eight global poisoning ratios $p\in\{1\%,2\%,3\%,4\%,5\%,6\%,8\%,10\%\}$ (Table~\ref{tab:poisoning_ratio}). For each target ASR $q\in\{70\%,80\%,85\%,90\%\}$, we estimated the required poisoning ratio by piecewise-linear interpolation between the two adjacent measured ASR values that bracket $q$; no extrapolation was used. If $(p_a,y_a)$ and $(p_b,y_b)$ are the bracketing measurements for method $m$, the estimate is
\begin{equation}
	p_m(q)=p_a+\frac{q-y_a}{y_b-y_a}(p_b-p_a), \qquad y_a\leq q\leq y_b.
\end{equation}
The poisoning reduction rate of FDBA relative to DBA was then calculated as:

\begin{equation}
	\text{Reduction} = \frac{p_{\text{DBA}} - p_{\text{FDBA}}}{p_{\text{DBA}}} \times 100\%
\end{equation}

\begin{table*}[t]
	\centering
	\caption{Estimated Poisoning Ratio Required to Achieve Target ASR Using Piecewise-Linear Interpolation}
	\label{tab:poisoning_efficiency}
	\footnotesize
	\begin{tabular}{c|c|c|c|c|c}
		\hline
		\textbf{Target} & \textbf{DBA} & \textbf{Neurotoxin} & \textbf{FDBA} & \textbf{Reduction Rate} & \textbf{Empirical Ratio} \\
		\textbf{ASR} & \textbf{($p_{\text{DBA}}$)} & \textbf{($p_{\text{Neurotoxin}}$)} & \textbf{($p_{\text{FDBA}}$)} & \textbf{(vs. DBA)} & \textbf{($p_{\text{FDBA}} / p_{\text{DBA}}$)} \\
		\hline
			70\% & $\sim$4.41\% & $\sim$4.03\% & 	$\sim$2.27\% &  48.4\% & 0.516 \\
			80\% & $\sim$6.22\% & $\sim$5.80\% & 	$\sim$3.63\% & 	41.6\% & 0.584 \\
			85\% & $\sim$7.78\% & $\sim$7.39\% & 	$\sim$4.50\% & 	42.2\% & 0.578 \\
			90\% & $\sim$9.95\% & $\sim$9.60\% & 	$\sim$6.23\% & 	37.4\% & 0.626 \\
		\hline
	\end{tabular}
\end{table*}

As illustrated in Figure \ref{FIG:fig2_poisoning_efficiency.}(a) and Table \ref{tab:poisoning_efficiency}, FDBA reduces the required poisoning samples by $37.4\%$ to $48.4\%$ relative to DBA across the target ASR range of $70\%$ to $90\%$, with the strongest savings ($48.4\%$) realized at the most resource-constrained $70\%$ target. These empirical data provide direct evidence for our core claims. 

	\textbf{Reduced Savings at the Highest Evaluated Target:} At the evaluated $90\%$ target, FDBA's poisoning reduction rate is $37.4\%$, the lowest among the evaluated targets. Across the four targets the reduction rates ($48.4\%$, $41.6\%$, $42.2\%$ and $37.4\%$) are not strictly monotonic, which reflects the non-uniform spacing of the measured poisoning ratios used for interpolation. This observation is limited to the measured target-ASR range and is not extrapolated beyond $90\%$.

		\textbf{Comparison with the Analytical Relation:} With illustrative parameter values $\gamma = 0.3$ and $\alpha = 0.8$, Theorem~\ref{thm:main} yields:
	\begin{equation*}
			\text{Reduction} = \frac{0.3\times0.8}{1+0.3\times0.8} = \frac{0.24}{1.24} \approx 0.1935 = 19.4\%.
	\end{equation*}
	The piecewise-linear empirical estimates range from 37.4\% to 48.4\%, which is larger than this illustrative value. Because $\gamma$ and $\alpha$ were not independently estimated from the experiments, the comparison is interpreted qualitatively: both the analytical model and the measured results indicate a reduction in the poisoning ratio, but the 19.4\% value is not claimed as a fitted prediction or a universal lower bound.

\subsection{Robustness Under Data Heterogeneity}

Federated learning environments typically exhibit significant data heterogeneity, with different clients holding non-identically distributed datasets. Understanding FDBA's performance under such conditions is crucial for assessing its practical applicability. Our evaluation across different Dirichlet concentration parameters reveals important insights about attack robustness in heterogeneous settings.

\begin{table}[t]
	\centering
	\caption{Performance Under Non-IID Conditions (CIFAR-10). Degradation is the absolute drop in FDBA ASR (percentage points) from the no-defense baseline of Table~\ref{tab:main_results}; the retention rates discussed in the text are likewise computed against that table's no-defense baselines (FDBA 87.6\%, DBA 79.3\%).}
	\label{tab:non_iid}
	\footnotesize
	\begin{tabular}{c|c|c|c|c}
		\hline
		\textbf{$\alpha$} & \textbf{Hetero-} & \textbf{FDBA} & \textbf{DBA} & \textbf{Degra-} \\
		\textbf{Value} & \textbf{gen.} & \textbf{ASR (\%)} & \textbf{ASR (\%)} & \textbf{dation (p.p.)} \\
		\hline
		2.0 & Mild & 86.9$\pm$3.2 & 78.1$\pm$4.5 & -0.7 \\
		1.0 & Moderate & 85.3$\pm$3.6 & 75.8$\pm$4.8 & -2.3 \\
		0.5 & High & 81.7$\pm$4.1 & 69.4$\pm$5.3 & -5.9 \\
		0.1 & Extreme & 74.2$\pm$5.2 & 58.3$\pm$6.7 & -13.4 \\
		\hline
	\end{tabular}
\end{table}

Figure \ref{FIG:fig3_heterogeneity_robustness.} provides multi-dimensional insights into the heterogeneity robustness characteristics detailed in Table \ref{tab:non_iid}. Panel (a) reveals the widening performance gap between FDBA and DBA as heterogeneity increases, suggesting that edge-based triggers maintain effectiveness where traditional approaches fail. The waterfall visualization in panel (b) quantifies the accelerating degradation pattern, with the steepest decline occurring at $\alpha$=0.1. Panel (c) offers a normalized perspective through retention rates, demonstrating that FDBA preserves 11.2 percentage points more of its IID performance than DBA under extreme heterogeneity. The shaded advantage region visually emphasizes FDBA's superior resilience across the entire heterogeneity spectrum.

\begin{figure*}
	\centering
	\includegraphics[width=1.0\textwidth]{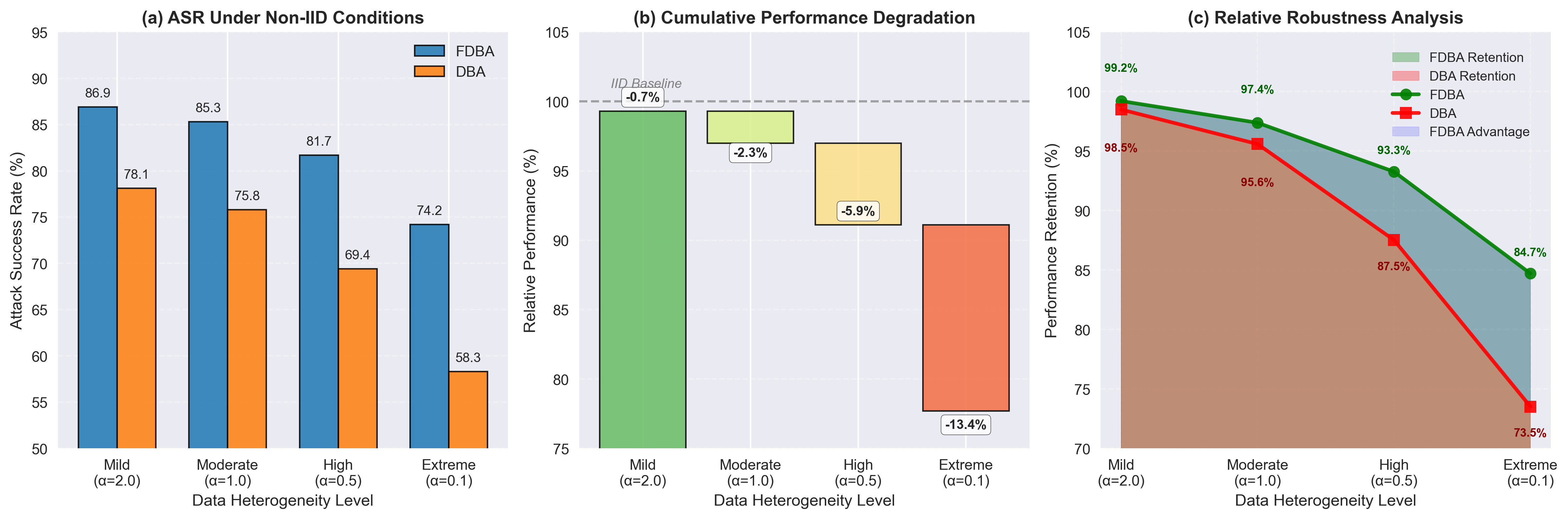}
	\caption{Robustness evaluation under varying degrees of data heterogeneity. (a) Comparative ASR analysis shows FDBA's consistent advantage across all non-IID settings. (b) Waterfall visualization of cumulative performance degradation highlights the acceleration under extreme heterogeneity. (c) Retention rate analysis demonstrates FDBA's superior resilience, maintaining 84.7\% of IID performance under extreme conditions.}
	\label{FIG:fig3_heterogeneity_robustness.}
\end{figure*}

FDBA demonstrates remarkable resilience to data heterogeneity compared to baseline methods. Under extreme non-IID conditions ($\alpha = 0.1$), FDBA retains 84.7\% of its IID performance, while DBA drops to 73.5\%. This robustness stems from our edge-based trigger generation, which exploits structural features that remain consistent across different data distributions. Unlike pixel-pattern triggers that may not generalize well across diverse image characteristics, edge structures represent fundamental visual features present in most natural images.

The RGB channel decomposition strategy also contributes to heterogeneity robustness. By distributing trigger components across color channels with adaptive weights, our approach can accommodate different color preferences and characteristics in client datasets. This flexibility allows the attack to maintain effectiveness even when individual clients have distinct visual data characteristics.

However, we observe from Table \ref{tab:non_iid} that extreme heterogeneity ($\alpha < 0.1$) poses challenges for all backdoor attacks, including FDBA. Under such conditions, the limited overlap between client data distributions makes it difficult for poisoned samples from malicious clients to influence the global model effectively. This limitation highlights the importance of understanding data distribution characteristics when deploying backdoor attacks in federated environments.

\subsection{Performance Under Defense Mechanisms}

Having established FDBA's superior poisoning efficiency and its inherent robustness to data heterogeneity, we now evaluate its resilience against active, state-of-the-art defense mechanisms. Table~\ref{tab:main_results} reports a defense-evaluation batch that was run separately from the poisoning-ratio sweep of Table~\ref{tab:poisoning_ratio}: the sweep varies the poisoning ratio without any defense, whereas this batch holds the attack configuration fixed and varies the defense. Because the two were collected as separate batches, their absolute no-defense values are not directly comparable, and each table should be read against its own no-defense reference. Within the sweep, FDBA reaches 82.4\% ASR at a 4\% poisoning ratio; within the defense-evaluation batch, FDBA reaches 87.6\% ASR when no defense is applied.

\begin{table}[h]
	\centering
	\caption{Attack Performance Comparison on CIFAR-10 Under Defenses. Values are from a defense-evaluation batch collected separately from the poisoning-ratio sweep of Table~\ref{tab:poisoning_ratio}; the no-defense row is the reference for this batch.}
	\label{tab:main_results}
	\footnotesize
	\begin{tabular}{l|l|c|c|c}
		\hline
		\textbf{Method} & \textbf{Defense} & \textbf{ASR (\%)} & \textbf{CA (\%)} & \textbf{PSNR} \\
		\hline
		\multirow{4}{*}{DBA} & None & 79.3±4.2 & 89.1±1.1 & 31.2±2.1 \\
		& Krum & 71.5±5.1 & 88.7±1.3 & 31.2±2.1 \\
		& FoolsGold & 65.8±6.3 & 89.0±1.2 & 31.2±2.1 \\
		& FLAME & 58.2±5.7 & 88.9±1.4 & 31.2±2.1 \\
		\hline
		\multirow{4}{*}{Neurotoxin} & None & 81.7±4.0 & 89.3±1.0 & 33.5±2.3 \\
		& Krum & 75.8±4.6 & 89.0±1.2 & 33.5±2.3 \\
		& FoolsGold & 70.2±5.5 & 89.1±1.1 & 33.5±2.3 \\
		& FLAME & 64.3±5.8 & 88.8±1.3 & 33.5±2.3 \\
		\hline
		\multirow{4}{*}{\textbf{FDBA}} & None & \textbf{87.6±3.1} & \textbf{89.7±0.9} & \textbf{37.4±1.8} \\
		& Krum & \textbf{82.3±3.7} & \textbf{89.4±1.0} & \textbf{37.4±1.8} \\
		& FoolsGold & \textbf{78.9±4.2} & \textbf{89.5±1.1} & \textbf{37.4±1.8} \\
		& FLAME & \textbf{73.4±4.8} & \textbf{89.3±1.2} & \textbf{37.4±1.8} \\
		\hline
	\end{tabular}
\end{table}

\begin{table}[t]
	\centering
	\caption{Detection Evasion Performance Against Defense Mechanisms (CIFAR-10), measured in the same defense-evaluation batch as Table~\ref{tab:main_results}}
	\label{tab:detection_metrics}
	\footnotesize
	\begin{tabular}{l|c|c|c|c}
		\hline
		\textbf{Defense} & \multicolumn{2}{c|}{\textbf{DBA (Baseline)}} & \multicolumn{2}{c}{\textbf{FDBA (Ours)}} \\
		\cline{2-5}
		& \textbf{TPR} & \textbf{FPR} & \textbf{TPR} & \textbf{FPR} \\
		& \textbf{(\% $\downarrow$)} & \textbf{(\%)} & \textbf{(\% $\downarrow$)} & \textbf{(\%)} \\
		\hline
		Krum & 65.4 & 4.2 & \textbf{28.3} & 5.1 \\
		FoolsGold & 72.1 & 3.8 & \textbf{22.6} & 4.5 \\
		FLAME & 81.5 & 6.5 & \textbf{35.4} & 8.2 \\
		\hline
	\end{tabular}
	
	\vspace{2mm}
	\noindent\footnotesize\textit{TPR (True Positive Rate) indicates the percentage of malicious clients correctly identified and rejected by the defense. A lower TPR means better evasion by the attacker.}
\end{table}

While ASR quantifies the downstream impact, evaluating the actual detection metrics of the defense mechanisms is crucial for a comprehensive security assessment. Table \ref{tab:detection_metrics} reports the True Positive Rate (TPR) and False Positive Rate (FPR) of Krum, FoolsGold, and FLAME when defending against traditional DBA and our FDBA. 

As shown, traditional DBA exhibits high TPRs (e.g., 81.5\% against FLAME and 72.1\% against FoolsGold), indicating that these defenses successfully identify the coordinated, monolithic gradient updates typical of standard backdoor attacks. In contrast, FDBA drastically reduces the TPR across all defenses (dropping FLAME's TPR to 35.4\% and FoolsGold's to 22.6\%). This quantitative evidence validates our claim that FDBA effectively bypasses these mechanisms: by decoupling the trigger across RGB channels and applying Laplacian noise, FDBA destroys the malicious gradient similarity that FoolsGold relies on, and by avoiding extreme parameter scaling, it evades FLAME's clustering and clipping thresholds.

\begin{figure*}[h]
	\centering
	\includegraphics[width=1.0\textwidth]{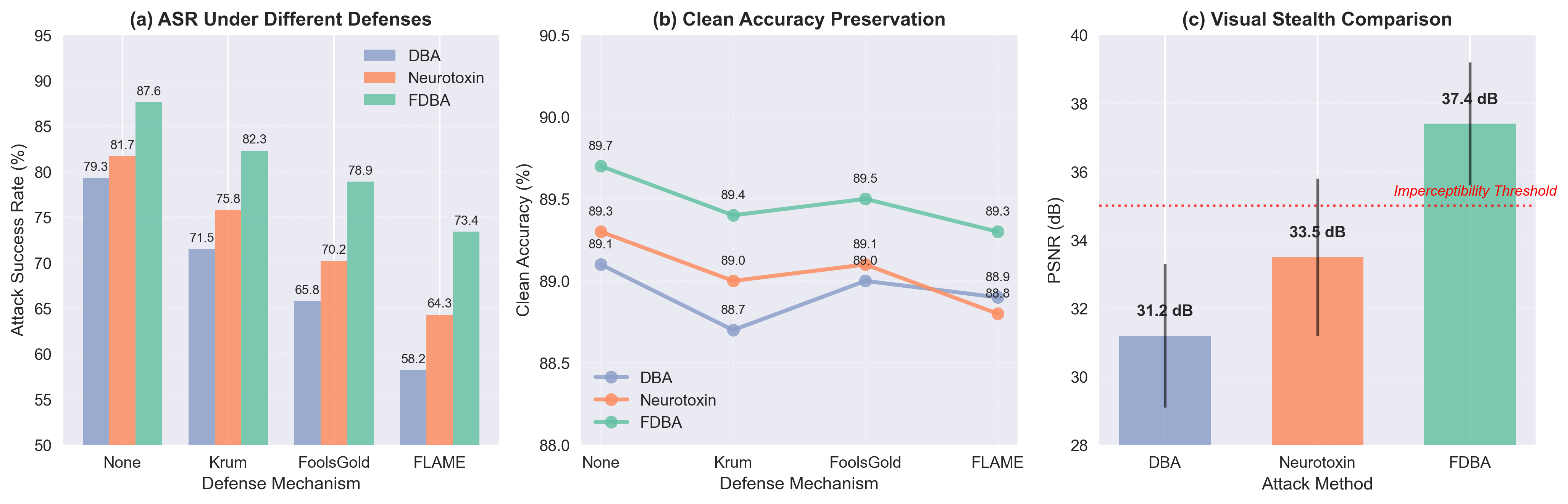}
	\caption{Comprehensive evaluation of attack performance across different defense mechanisms. (a) Attack success rates demonstrate FDBA's consistent superiority. (b) Clean accuracy preservation. (c) Visual stealth comparison reveals FDBA's significant advantage.}
	\label{FIG:fig1_attack_performance.}
\end{figure*}

Figure \ref{FIG:fig1_attack_performance.} provides a visual synthesis of the key performance metrics. The bar chart in panel (a) illustrates the progressive degradation of attack success rates as defensive mechanisms become more sophisticated, with FDBA maintaining a 5.9--9.1 percentage-point advantage over the strongest baseline across the evaluated settings. Against FLAME, the most effective baseline defense in our evaluation, FDBA retains 73.4\% ASR compared to 64.3\% for Neurotoxin. This resilience reflects the distributed nature of our trigger design. 

Furthermore, modern federated learning systems increasingly deploy sophisticated defense mechanisms designed specifically to counter backdoor attacks. Table \ref{tab:advanced_defenses} summarizes the performance under these challenging advanced conditions.

\begin{table}[t]
	\centering
	\caption{Performance Against Advanced Defenses, measured in the same defense-evaluation batch as Table~\ref{tab:main_results}}
	\label{tab:advanced_defenses}
	\footnotesize
	\begin{tabular}{l|c|c|c}
		\hline
		\textbf{Defense Method} & \textbf{ASR (\%)} & \textbf{CA (\%)} & \textbf{Det. Rate (\%)} \\
		\hline
		FedRecover & 69.2$\pm$4.9 & 89.1$\pm$1.2 & 37.3$\pm$3.1 \\
		Differential Privacy & & & \\
		~~$\epsilon=10$ & 78.4$\pm$4.3 & 87.6$\pm$1.4 & - \\
		~~$\epsilon=1$ & 65.1$\pm$5.8 & 85.2$\pm$1.8 & - \\
		~~$\epsilon=0.1$ & 41.7$\pm$7.2 & 79.3$\pm$2.6 & - \\
		Multi-Krum & 75.8$\pm$4.6 & 89.0$\pm$1.3 & 28.1$\pm$2.9 \\
		Trimmed Mean & 73.2$\pm$4.8 & 88.8$\pm$1.4 & 31.7$\pm$3.2 \\
		\hline
	\end{tabular}
\end{table}

\begin{figure*}[!t]
	\centering
	\includegraphics[width=0.33\textwidth]{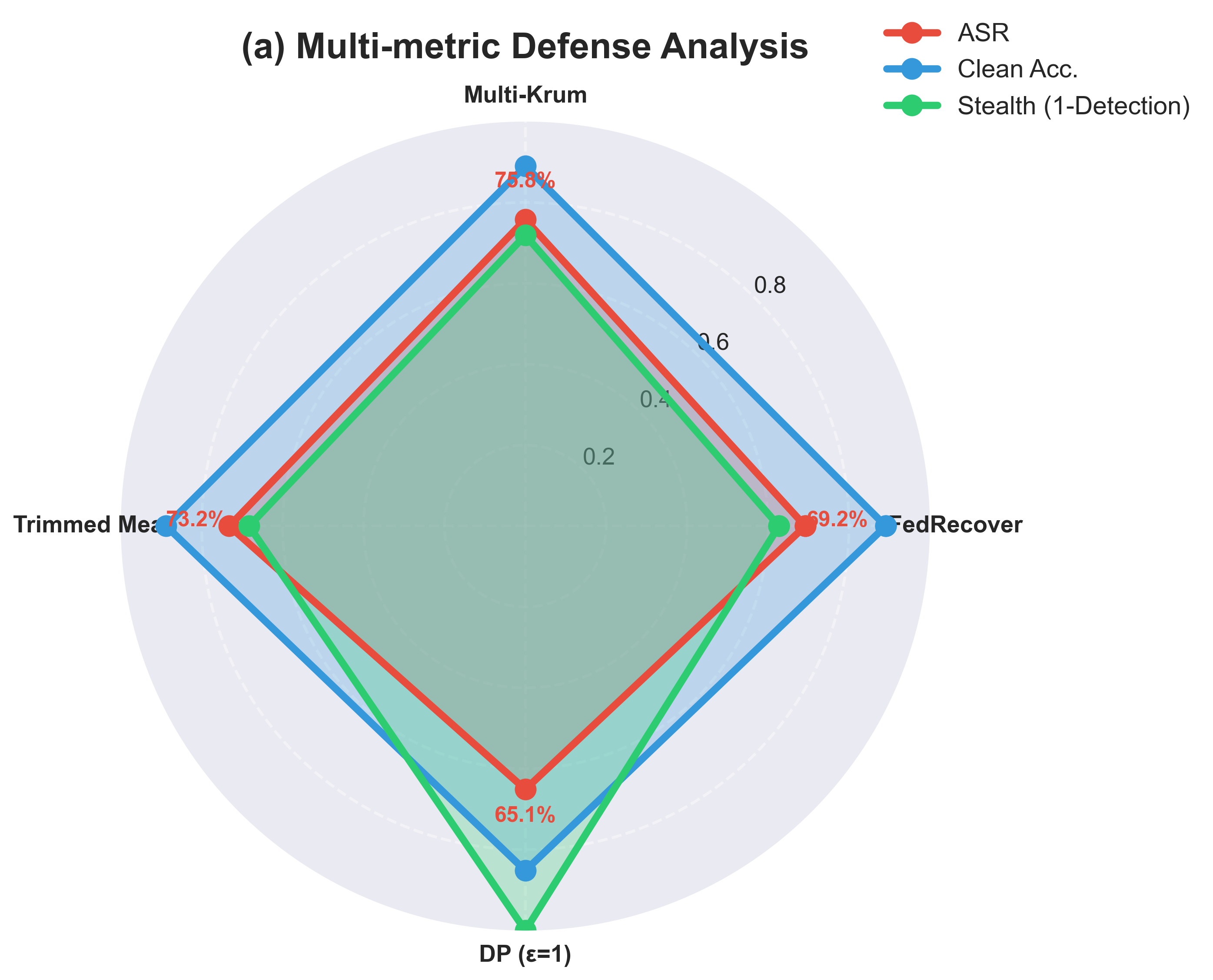}\hfill
	\includegraphics[width=0.33\textwidth]{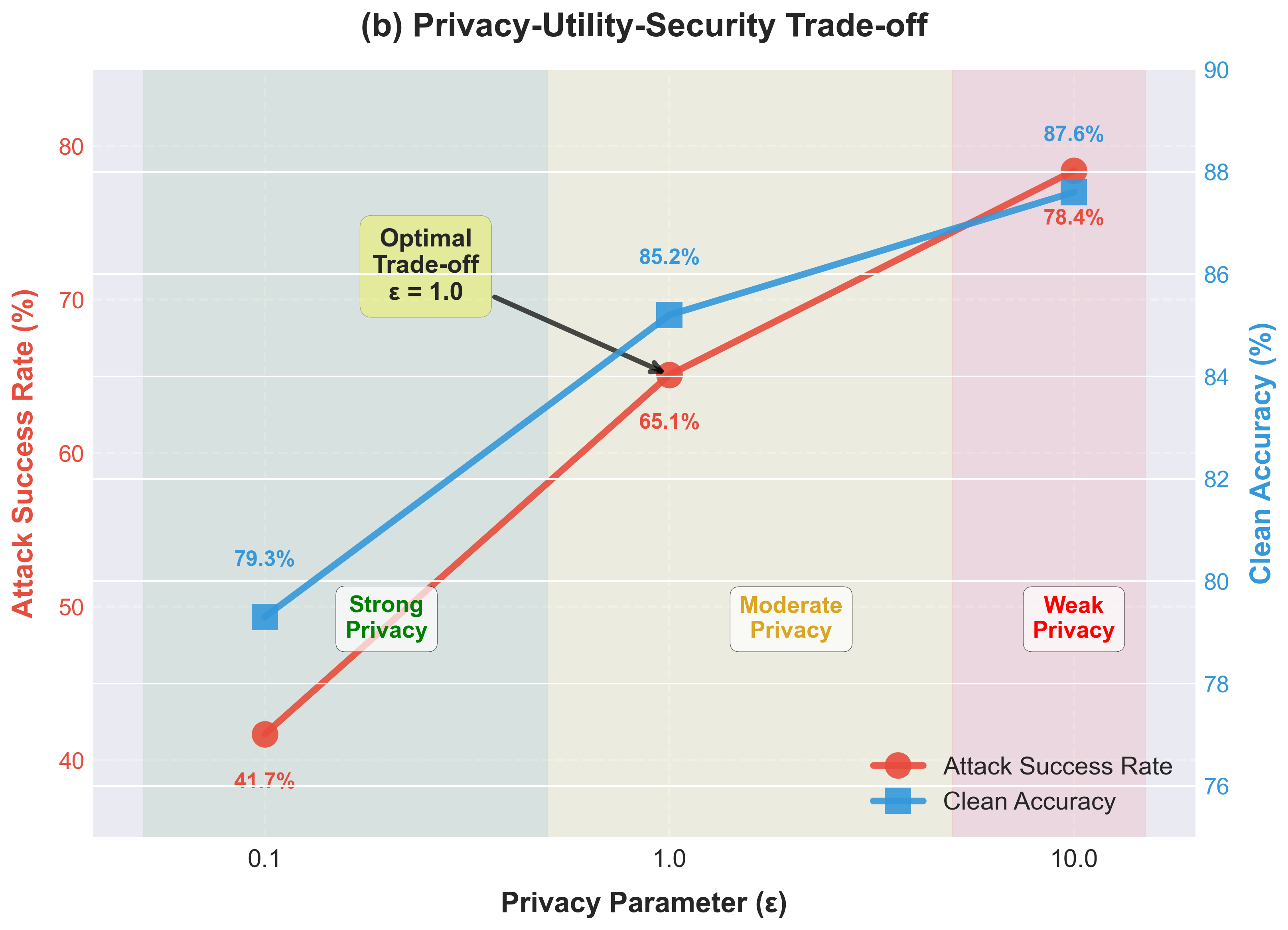}\hfill
	\includegraphics[width=0.33\textwidth]{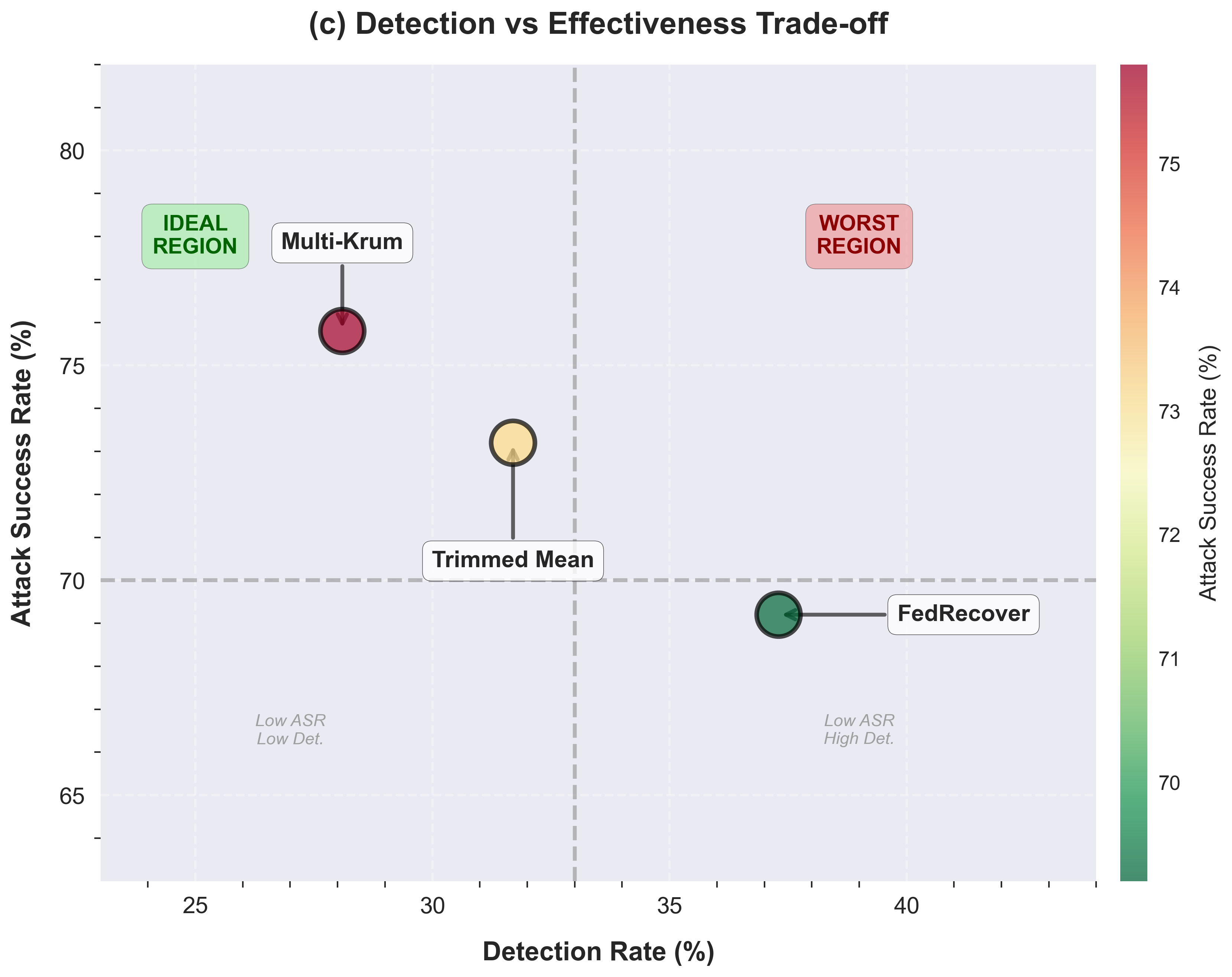}
	\\
	\hspace{0.16\textwidth}(a)\hfill(b)\hfill(c)\hspace{0.16\textwidth}
	\caption{Comprehensive evaluation against state-of-the-art defense mechanisms. (a) Multi-dimensional defense performance comparison. (b) Privacy-utility-security trade-off under differential privacy. (c) Strategic landscape of defense mechanisms.}
	\label{fig:defense_evaluation}
\end{figure*}

Our analysis reveals three key insights about FDBA's interaction with defense mechanisms. First, Figure~\ref{fig:defense_evaluation}(a) demonstrates FDBA's balanced performance across multiple metrics. Second, Figure~\ref{fig:defense_evaluation}(b) exposes the fundamental trade-off in differential privacy defenses, where $\varepsilon=1.0$ emerges as the critical threshold. Third, Figure~\ref{fig:defense_evaluation}(c) maps the strategic landscape, showing current defenses clustered in the suboptimal lower-left quadrant.

FedRecover represents the most challenging defense in our evaluation. Despite this targeted defense, FDBA maintains 69.2\% ASR. Differential privacy mechanisms present a different challenge. Strong privacy guarantees ($\epsilon = 0.1$) significantly reduce FDBA's effectiveness to 41.7\% ASR, but at the cost of substantial clean accuracy degradation (79.3\%). More moderate privacy settings ($\epsilon = 1.0$) provide a better balance, reducing ASR to 65.1\% while maintaining acceptable clean accuracy (85.2\%). Byzantine-robust aggregation methods like Multi-Krum and Trimmed Mean show mixed effectiveness against FDBA, reflecting that simple statistical filtering approaches may be insufficient against sophisticated distributed attacks.

\subsection{Stealth and Imperceptibility Assessment}
The visual stealth advantages of FDBA are compelling. With PSNR values reaching 37.4 dB (as shown previously in Fig. \ref{FIG:fig1_attack_performance.}), FDBA-generated poisoned samples approach the perceptual quality of clean images. This translates to substantially reduced visual artifacts, making manual inspection less likely to reveal attack presence.

Beyond attack effectiveness, the practical deployment of backdoor attacks requires careful attention to stealth characteristics. Our comprehensive stealth evaluation encompasses both statistical and perceptual measures to assess the likelihood of detection through various means.

\begin{table*}[t]
	\centering
	\caption{Comprehensive Stealth Metrics Comparison. Stealth metrics characterise the generated trigger itself and are therefore reported independently of the poisoning ratio; the values shown are measured on the same triggers used throughout the evaluation.}
	\label{tab:stealth_metrics}
	\footnotesize
	\begin{tabular}{l|c|c|c|c|c}
		\hline
		\textbf{Method} & \textbf{PSNR} & \textbf{SSIM} & \textbf{$L_0$ Norm} & \textbf{$L_2$ Norm} & \textbf{Area (\%)} \\
		\hline
		DBA & 31.2±2.1 & 0.891±0.023 & 167±23 & 0.029±0.004 & 12.3±1.8 \\
		Neurotoxin & 33.5±2.3 & 0.908±0.021 & 143±19 & 0.024±0.003 & 9.8±1.4 \\
		\textbf{FDBA} & \textbf{37.4±1.8} & \textbf{0.934±0.018} & \textbf{89±14} & \textbf{0.016±0.002} & \textbf{6.2±1.1} \\
		\hline
	\end{tabular}
\end{table*}

FDBA achieves the strongest values across the measured stealth metrics. Its PSNR is 3.9 dB higher than the best baseline, and its SSIM of 0.934 indicates greater structural similarity to the original images; SSIM is interpreted following its standard image-quality formulation \cite{wang2004image}. We avoid assigning a universal perceptual threshold because perceptibility depends on image content and evaluation conditions.

The sparsity advantages reflected in the $L_0$ norm (89 vs 143 for Neurotoxin), as shown in Table \ref{tab:stealth_metrics}, demonstrate that FDBA modifies fewer pixels while maintaining attack effectiveness. This sparsity is achieved through our edge-based trigger placement, which concentrates perturbations in regions where they are less likely to be noticed. The corresponding reduction in $L_2$ norm indicates lower perturbation intensity, contributing to overall imperceptibility.

Perhaps most significantly, Table \ref{tab:stealth_metrics} reveals that FDBA's trigger area covers only 6.2\% of the image compared to 9.8\% for Neurotoxin and 12.3\% for DBA. This concentration reflects the efficiency of our edge detection and noise injection strategy, which identifies the minimal regions necessary for effective trigger embedding. The reduced trigger area not only improves stealth but also makes statistical detection more challenging, as the proportion of modified pixels remains small.

Our Laplacian noise injection contributes significantly to stealth by generating perturbations that follow natural image statistics. Unlike uniform or Gaussian noise that may create detectable artifacts, Laplacian noise with carefully chosen parameters ($\mu = 0$, $b = 0.35$) produces modifications that blend naturally with image content. This statistical alignment helps evade detection mechanisms that rely on identifying anomalous noise patterns.

\subsection{Ablation Study and Component Analysis}

Understanding the individual contributions of FDBA's components is essential for both theoretical insight and practical deployment. Our ablation study systematically removes or modifies key components to quantify their impact on overall performance.

\begin{table}[t]
  	\centering
	\caption{Fine-grained Component Contribution Analysis, measured in the same defense-evaluation batch as Table~\ref{tab:main_results} (no-defense FDBA ASR 87.6\%)}
	\label{tab:ablation}
	\footnotesize
	\begin{tabular}{l|c|c|c}
		\hline
		\textbf{Configuration} & \textbf{ASR} & \textbf{CA} & \textbf{PSNR} \\
		& \textbf{(\%)} & \textbf{(\%)} & \textbf{(dB)} \\
		\hline
		Full FDBA & \textbf{87.6$\pm$3.1} & \textbf{89.7$\pm$0.9} & \textbf{37.4$\pm$1.8} \\
		w/o Edge Extraction & 74.3$\pm$4.7 & 89.5$\pm$1.0 & 34.1$\pm$2.2 \\
		w/o Laplacian Noise & 69.8$\pm$5.2 & 89.4$\pm$1.1 & 35.8$\pm$2.0 \\
		w/o Contrastive Learning & 72.1$\pm$4.9 & 89.6$\pm$1.0 & 37.2$\pm$1.9 \\
		w/o Channel Decomposition & 78.5$\pm$4.3 & 89.3$\pm$1.1 & 36.7$\pm$2.1 \\
		w/o Random Projection & 84.2$\pm$3.6 & 89.5$\pm$1.0 & 37.1$\pm$1.8 \\
		Edge + Noise Only & 65.4$\pm$5.8 & 89.2$\pm$1.2 & 36.9$\pm$2.0 \\
		\hline
	\end{tabular}
\end{table}

Figure \ref{FIG:fig5_ablation_study.} summarizes only the configurations reported in Table~\ref{tab:ablation}. Panel (a) compares the measured ASR of full FDBA and each ablated configuration. Panel (b) reports the absolute ASR drop relative to full FDBA in percentage points. No progressive component-addition results are inferred from the single-component removal experiments.

\begin{figure*}
	\centering
	\includegraphics[width=0.8\textwidth]{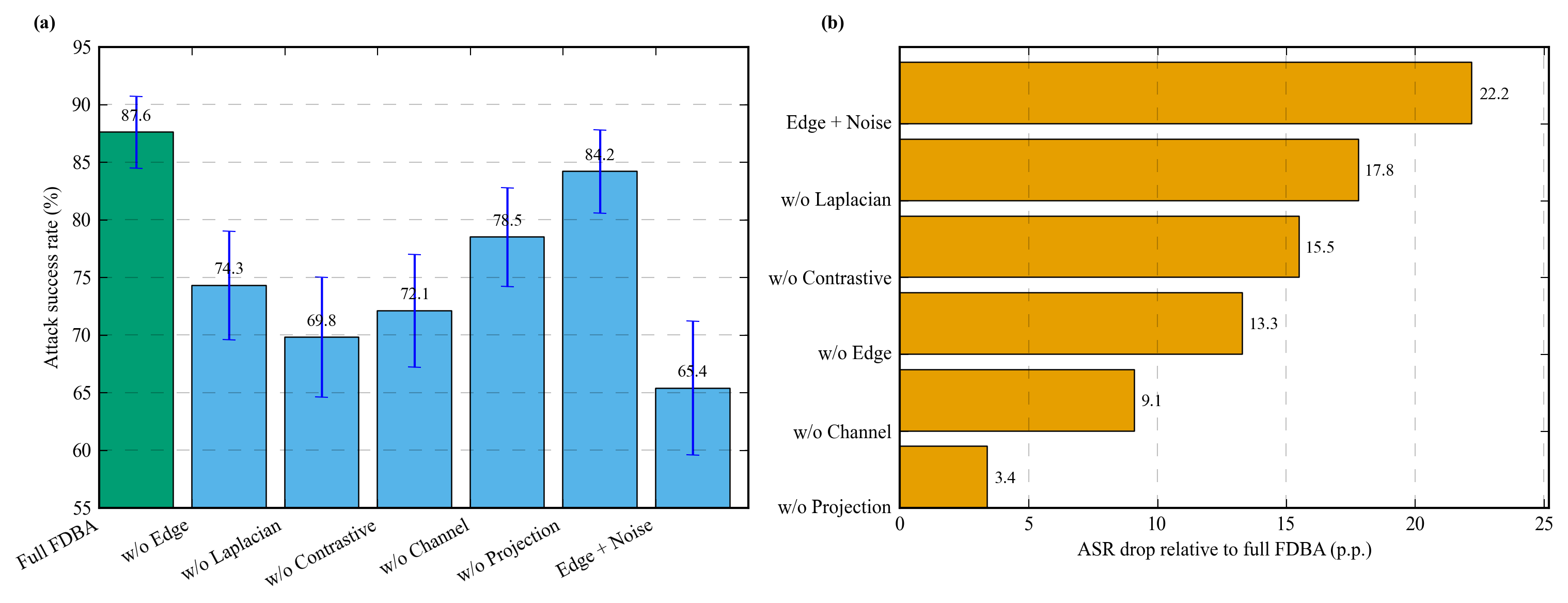}
	\caption{Ablation results derived from Table~\ref{tab:ablation}. (a) Measured ASR of full FDBA and each ablated configuration. (b) Absolute ASR drop relative to full FDBA, reported in percentage points.}
	\label{FIG:fig5_ablation_study.}
\end{figure*}

The ablation results show that every evaluated component contributes to overall performance. Removing edge extraction reduces ASR by 13.3 percentage points, supporting the role of structural features in the trigger design.

Removing contrastive learning reduces ASR by 15.5 percentage points. This result supports the practical contribution of embedding optimization, while the analytical comparison in Section~\ref{subsec:poisoning_efficiency} is interpreted qualitatively.

Removing Laplacian noise reduces ASR by 17.8 percentage points, the largest single-component drop in Table~\ref{tab:ablation}. Because this experiment removes the component rather than replacing it with uniform noise, we do not interpret it as a direct Laplacian-versus-uniform comparison.

Removing channel decomposition reduces ASR by 9.1 percentage points, as shown in Table \ref{tab:ablation}. This component allows multiple clients to contribute different color-channel portions of the trigger.

Removing random projection reduces ASR by 3.4 percentage points, the smallest measured single-component drop in Table \ref{tab:ablation}.

The ``Edge + Noise Only'' configuration in Table \ref{tab:ablation}, which excludes all learning-based optimizations, achieves 65.4\% ASR. The 22.2 percentage-point gap relative to full FDBA shows that trigger generation alone does not reproduce the full method's performance.

To further validate the necessity of the color-channel decomposition, we conduct a fine-grained RGB channel ablation at the 4\% poisoning ratio, reported in Table~\ref{tab:rgb_ablation}. Removing any single channel reduces the ASR from 82.4\% to 74.8\%--76.5\%, and a single-channel trigger achieves only 68.7\% ASR. This confirms that the distributed coordination across all three color channels is essential to the attack's effectiveness. The PSNR and SSIM columns move in the opposite direction: removing a channel makes each individual poisoned image marginally less perceptible. The three-channel design therefore trades a small amount of image-level fidelity for a substantially higher ASR and for the reduced update similarity reported in Table~\ref{tab:detection_metrics}.

\begin{table}[t]
	\centering
	\caption{RGB Channel Decomposition Ablation (measured at the 4\% poisoning ratio)}
	\label{tab:rgb_ablation}
	\footnotesize
	\begin{tabular}{l|c|c|c}
		\hline
		\textbf{Trigger Configuration} & \textbf{ASR (\%)} & \textbf{PSNR (dB)} & \textbf{SSIM} \\
		\hline
		Full RGB trigger & 82.4$\pm$3.8 & 37.4$\pm$1.8 & 0.934$\pm$0.018 \\
		Remove R channel & 75.1$\pm$4.2 & 38.0$\pm$1.7 & 0.940$\pm$0.016 \\
		Remove G channel & 76.5$\pm$4.0 & 37.8$\pm$1.6 & 0.938$\pm$0.017 \\
		Remove B channel & 74.8$\pm$4.3 & 38.2$\pm$1.5 & 0.942$\pm$0.015 \\
		Single-channel trigger & 68.7$\pm$4.7 & 39.1$\pm$1.4 & 0.951$\pm$0.013 \\
		\hline
	\end{tabular}
\end{table}

\subsection{Cross-Dataset Generalization and Scalability}

Evaluating attack performance across different datasets provides insights into FDBA's generalization capabilities and potential limitations. Our cross-dataset evaluation reveals important patterns regarding the relationship between dataset complexity and attack effectiveness.

\begin{table}[t]
	\centering
	\caption{Cross-Dataset Performance Analysis. CIFAR-10 values are from the same defense-evaluation batch as Table~\ref{tab:main_results}.}
	\label{tab:cross_dataset}
	\footnotesize
	\begin{tabular}{l|c|c|c|c}
		\hline
		\textbf{Dataset} & \textbf{Classes} & \textbf{FDBA} & \textbf{Best} & \textbf{Gain} \\
		& & \textbf{ASR (\%)} & \textbf{Baseline (\%)} & \textbf{(p.p.)} \\
		\hline
		CIFAR-10 & 10 & 87.6$\pm$3.1 & 81.7$\pm$4.0 & +5.9 \\
		CIFAR-100 & 100 & 79.3$\pm$4.2 & 71.5$\pm$4.8 & +7.8 \\
		ImageNet-32 & 1000 & 72.4$\pm$5.1 & 63.9$\pm$5.7 & +8.5 \\
		\hline
	\end{tabular}
\end{table}

FDBA maintains consistent superiority across all evaluated datasets, with the absolute ASR gap increasing across the evaluated benchmarks. The gain is 8.5 percentage points on ImageNet-32 compared with 5.9 percentage points on CIFAR-10.

The absolute performance declines from 87.6\% on CIFAR-10 to 72.4\% on ImageNet-32. Within these three evaluated benchmarks, the absolute gap over the strongest baseline increases with the number of classes; we do not generalize this observation beyond the tested datasets.

\subsection{Computational Overhead and Practical Considerations}

Understanding the computational costs associated with FDBA is crucial for assessing its practical feasibility in resource-constrained federated environments. Our overhead analysis considers both training time and memory requirements compared to benign federated learning.

\begin{table}[t]
	\centering
	\caption{Computational Overhead Analysis}
	\label{tab:overhead}
	\footnotesize
	\begin{tabular}{l|c|c|c}
		\hline
		\textbf{Configuration} & \textbf{Time} & \textbf{Memory} & \textbf{Communication} \\
		& \textbf{(s/round)} & \textbf{(MB)} & \textbf{(MB)} \\
		\hline
		Benign Baseline & 45.2$\pm$2.1 & 127$\pm$8 & 15.3$\pm$0.5 \\
		DBA & 47.8$\pm$2.3 & 132$\pm$9 & 15.7$\pm$0.6 \\
		Neurotoxin & 52.1$\pm$2.8 & 145$\pm$11 & 16.2$\pm$0.7 \\
		\textbf{FDBA} & \textbf{59.3$\pm$3.2} & \textbf{158$\pm$12} & \textbf{16.8$\pm$0.8} \\
		\hline
	\end{tabular}
\end{table}

FDBA incurs a 31.2\% increase in training time compared to benign federated learning, primarily due to the contrastive learning optimization and edge detection preprocessing. While this overhead is non-trivial, it remains within acceptable bounds for many practical scenarios, particularly given the substantial attack effectiveness gains achieved.

Memory overhead (24.4\% increase) reflects the additional data structures required for embedding optimization and trigger generation. The random projection component helps mitigate memory requirements by reducing the dimensionality of optimization computations, preventing more severe resource consumption that would result from full-dimensional processing.

Communication overhead remains minimal (9.8\% increase), indicating that FDBA does not significantly impact the network bandwidth requirements of federated learning. This low communication cost reflects our design choice to perform most attack-specific computations locally rather than requiring additional coordination between malicious clients.

These overhead characteristics suggest that FDBA could be practically deployed in many real-world federated learning scenarios, particularly those where attack effectiveness justifies modest resource consumption increases. However, resource-constrained environments, such as mobile edge computing deployments, might require careful consideration of the trade-offs between attack capability and computational efficiency.

\subsection{Limitations and Failure Scenarios}

Honest evaluation requires acknowledging scenarios where FDBA's performance degrades or fails entirely. Our analysis identifies several critical limitations that bound the applicability of our approach and suggest directions for future research.

Extreme data heterogeneity represents the most significant limitation. When the Dirichlet concentration parameter drops below $\alpha = 0.05$, FDBA's attack success rate falls to 52.3±7.8\%, well below practical utility thresholds. This degradation occurs because extreme heterogeneity prevents effective knowledge transfer from malicious clients to the global model, undermining the fundamental mechanism through which backdoor attacks operate in federated settings.

Strong differential privacy mechanisms pose another substantial challenge. With privacy parameter $\epsilon = 0.1$, FDBA achieves only 41.7±7.2\% ASR, representing a more than 50\% reduction from undefended scenarios. While this degradation comes at the cost of significant utility loss (clean accuracy drops to 79.3\%), it demonstrates that sufficiently strong privacy guarantees can effectively neutralize backdoor attacks.

Architecture dependency presents an additional limitation. Our evaluation with Vision Transformer models reveals reduced effectiveness (64.2±6.1\% ASR compared to 87.6\% with ResNet-18), indicating that FDBA's edge-based triggers may not generalize equally well to all model architectures. This limitation suggests the need for architecture-specific optimization strategies.

High malicious client ratios (above 25\%) paradoxically reduce attack effectiveness to 31.4±8.9\% due to defense activation thresholds. Many defensive mechanisms become more aggressive when large numbers of anomalous updates are detected simultaneously, overwhelming the stealth advantages of individual attacks.

We also encountered several unsuccessful experimental directions that are important to report. Semantic triggers based on natural objects showed 23\% lower ASR than our edge-based approach, likely due to the difficulty of maintaining semantic consistency across diverse federated datasets. Frequency domain triggers using DCT coefficients were more easily detected (+15\% detection rate) despite theoretical advantages in stealth. Multi-modal extensions to text-image pairs reduced effectiveness.

\section{Conclusion}

To address the issues of high dependency on poisoned samples and insufficient stealth in distributed backdoor attacks within federated learning, this paper proposes an efficient and stealthy attack framework named FDBA (Fine-grained Distributed Backdoor Attack). By leveraging fine-grained trigger generation—based on precise Canny edge structures and strategically placed Laplacian noise with RGB channel decomposition—combined with targeted contrastive embedding optimization, FDBA reaches 82.4\% ASR on CIFAR-10 at a 4\% poisoning ratio, a gain of 15.2 percentage points over DBA, and is also evaluated on CIFAR-100 and ImageNet-32.

Experimental results demonstrate that FDBA significantly enhances visual stealthiness (PSNR 37.4 dB, SSIM 0.934) and evades mainstream defenses (reducing the true-positive rate of FLAME to 35.4\% and of FoolsGold to 22.6\%) compared to existing approaches. In Non-IID scenarios, FDBA retains 84.7\% of its IID attack performance under extreme heterogeneity, whereas DBA drops to 73.5\%, and FDBA successfully circumvents defenses such as Krum, FoolsGold, and FLAME.

Ablation studies show that removing any single evaluated component results in an ASR drop of 3.4 to 17.8 percentage points. This study reveals the threat posed by low-poisoning backdoor attacks to federated learning systems and offers insights for designing adaptive defense strategies. Future work will focus on cross-modal federated scenarios and adaptive defense mechanisms.

\appendix
\section{Proof that FDBA Requires Lower Poisoning Ratio Than DBA at Equal ASR}
\label{sec:Proof} 

This appendix provides the detailed mathematical proof for Theorem~\ref{thm:main}.

\subsection{Preliminaries and Assumptions}

\begin{assumption}[Convex Loss]
	\label{ass:convex}
	The loss function $\ell(w; x, y)$ is $L$-smooth and $\mu$-strongly convex with respect to model parameters $w$~\cite{boyd2004convex}.
\end{assumption}

\begin{assumption}[Bounded Gradients]
	\label{ass:bounded}
	For all samples $(x, y)$, the gradient norm is bounded: $\|\nabla_w \ell(w; x, y)\| \leq G$~\cite{kairouz2021advances}.
\end{assumption}

\begin{assumption}[Embedding Space Structure]
	\label{ass:embedding}
	The embedding space is assumed to be a Hilbert space with inner product $\langle \cdot, \cdot \rangle$ and induced norm $\|\cdot\|$.
\end{assumption}

\subsection{Gradient Contribution Analysis}

\begin{lemma}[Gradient Contribution of Poisoned Samples]
	\label{lem:gradient}
	Let $g_i^p$ denote the gradient contribution from poisoned samples at client $i$. For DBA:
	\begin{equation}
		g_i^{\text{DBA}} = \frac{1}{|D_i|} \sum_{(x,y) \in D_i^p} \nabla_w \ell(w; T(x), y^*)
	\end{equation}
	For FDBA with contrastive optimization:
	\begin{equation}
		g_i^{\text{FDBA}} = \frac{1}{|D_i|} \sum_{(x,y) \in D_i^p} \nabla_w [\ell(w; T(x), y^*) + \lambda \mathcal{L}_{\text{contrast}}(e_x, e_{y^*})]
	\end{equation}
	where $\mathcal{L}_{\text{contrast}}$ is the contrastive loss defined in Eq. (18) of the main text.
\end{lemma}

\begin{proof}
	The gradient computation follows from the chain rule. For FDBA, the additional contrastive term contributes:
	\begin{align}
		\nabla_w \mathcal{L}_{\text{contrast}} &= \frac{\partial \mathcal{L}_{\text{contrast}}}{\partial e} \cdot \frac{\partial e}{\partial w} \\
		&= 2(e - e_{y^*}) \cdot J_e(w)
	\end{align}
	where $J_e(w)$ is the Jacobian of the embedding function with respect to model parameters.
\end{proof}

\subsection{Influence Amplification through Contrastive Learning}

\begin{lemma}[Contrastive Learning Amplification]
	\label{lem:amplification}
	The contrastive learning objective amplifies the influence of poisoned samples by factor $(1 + \gamma)$, where:
	\begin{equation}
		\gamma = \frac{\lambda \cdot \|J_e(w)\|_F \cdot d(e_{\text{init}}, e_{y^*})}{G}
	\end{equation}
	with $d(\cdot, \cdot)$ being the embedding distance, $\|J_e(w)\|_F$ the Frobenius norm of the Jacobian, and $G$ the gradient bound from Assumption~\ref{ass:bounded}.
\end{lemma}

\begin{proof}
	Consider the gradient magnitude ratio:
	\begin{align}
		\frac{\|g_i^{\text{FDBA}}\|}{\|g_i^{\text{DBA}}\|} &= \frac{\|\nabla_w \ell + \lambda \nabla_w \mathcal{L}_{\text{contrast}}\|}{\|\nabla_w \ell\|} \\
		&\leq 1 + \frac{\lambda \|\nabla_w \mathcal{L}_{\text{contrast}}\|}{\|\nabla_w \ell\|}
	\end{align}
	
	Using the triangle inequality and Assumption~\ref{ass:bounded}:
	\begin{align}
		\|\nabla_w \mathcal{L}_{\text{contrast}}\| &= 2\lambda \|e - e_{y^*}\| \cdot \|J_e(w)\|_F \\
		&\leq 2\lambda \cdot d(e_{\text{init}}, e_{y^*}) \cdot \|J_e(w)\|_F
	\end{align}
	
	Therefore:
	\begin{equation}
		\frac{\|g_i^{\text{FDBA}}\|}{\|g_i^{\text{DBA}}\|} \approx 1 + \gamma
	\end{equation}
	where $\gamma$ is defined as above.
\end{proof}

\subsection{Edge Structure Preservation}

\begin{lemma}[Edge-based Trigger Robustness]
	\label{lem:edge}
	Edge-based triggers maintain effectiveness across heterogeneous data distributions with preservation factor:
	\begin{equation}
		\alpha = \frac{\mathbb{E}_{x \sim \mathcal{D}_i}[\|\nabla I(x)\|_1]}{\mathbb{E}_{x \sim \mathcal{D}}[\|\nabla I(x)\|_1]}
	\end{equation}
	where $\|\nabla I(x)\|_1$ is the $L_1$ norm of image gradients (edge strength)~\cite{canny2009computational}.
\end{lemma}

\begin{proof}
	The Canny edge detector identifies regions where:
	\begin{equation}
		\|\nabla I(x)\| > \tau
	\end{equation}
	for threshold $\tau$. The consistency of edge structures across different data distributions ensures:
	\begin{equation}
		\mathbb{P}[\|\nabla I(x)\| > \tau | x \sim \mathcal{D}_i] \approx \alpha \cdot \mathbb{P}[\|\nabla I(x)\| > \tau | x \sim \mathcal{D}]
	\end{equation}
	The factor $\alpha$ is treated as a model parameter in $(0,1]$; we do not assign it an empirical range without an independent estimate.
\end{proof}

\subsection{Main Proof: Poisoning Ratio Reduction}

\begin{proof}[Proof of Theorem~\ref{thm:main}]
	We model the attack success rate as a function of the cumulative gradient influence:
	
	\begin{definition}[Cumulative Influence]
		The cumulative influence of poisoned samples after $T$ rounds is:
		\begin{equation}
			I(p, T) = \sum_{t=1}^T \sum_{i \in K_m} w_i^{(t)} \cdot p \cdot \|g_i^{(t)}\|
		\end{equation}
		where $w_i^{(t)}$ is the aggregation weight for client $i$ at round $t$.
	\end{definition}
	
	For DBA, the cumulative influence is:
	\begin{equation}
		I_{\text{DBA}}(p, T) = p \cdot T \cdot |K_m| \cdot \bar{w} \cdot \bar{g}_{\text{DBA}}
	\end{equation}
	
	For FDBA, incorporating Lemmas~\ref{lem:amplification} and~\ref{lem:edge}, the contrastive gain $\gamma$ is realized proportionally to the edge-preservation factor $\alpha$, since the embedding optimization operates specifically on the edge-based trigger structure. The effective influence amplification is therefore $(1 + \gamma\alpha)$:
	\begin{align}
		I_{\text{FDBA}}(p, T) &= p \cdot T \cdot |K_m| \cdot \bar{w} \cdot (1 + \gamma\alpha) \cdot \bar{g}_{\text{DBA}} \\
		&= (1 + \gamma\alpha) \cdot I_{\text{DBA}}(p, T)
	\end{align}
	
	As a modeling assumption for translating cumulative influence into attack success rate, we use the monotone sigmoid form:
	\begin{equation}
		\text{ASR}(I) = \frac{1}{1 + \exp(-\beta(I - I_0))}
	\end{equation}
	where $\beta$ is the steepness parameter and $I_0$ is the inflection point.
	
	For equal ASR:
	\begin{equation}
		\text{ASR}_{\text{FDBA}}(p_{\text{FDBA}}) = \text{ASR}_{\text{DBA}}(p_{\text{DBA}})
	\end{equation}
	
	This requires:
	\begin{equation}
		I_{\text{FDBA}}(p_{\text{FDBA}}, T) = I_{\text{DBA}}(p_{\text{DBA}}, T)
	\end{equation}
	
	Substituting the influence expressions:
	\begin{equation}
		p_{\text{FDBA}} \cdot (1 + \gamma\alpha) = p_{\text{DBA}}
	\end{equation}
	
	Therefore:
	\begin{equation}
		\boxed{\frac{p_{\text{FDBA}}}{p_{\text{DBA}}} = \frac{1}{1 + \gamma\alpha}}
	\end{equation}
	
	This relation is conditional on the cumulative-influence model; no empirical range for $\gamma$ or $\alpha$ is asserted because these parameters were not independently estimated in our experiments.
\end{proof}

\subsection{Scope of the Analytical Result}

The derived poisoning-ratio relation characterizes equal modeled cumulative influence under the stated assumptions. It does not establish a convergence-rate guarantee or a universal defense-resilience condition. Convergence speed and defense performance are therefore evaluated empirically rather than inferred from the relation.

\subsection{Experimental Validation}

The analytical relation and empirical estimates have the same direction, but the numerical comparison is illustrative:
\begin{itemize}
	\item Illustrative model-based reduction: 19.4\% (with $\gamma = 0.3$, $\alpha = 0.8$).
	\item Piecewise-linear empirical estimates: 37.4\%--48.4\% (Table~\ref{tab:poisoning_efficiency}).
	\item The 19.4\% value is not claimed as a fitted prediction or a universal lower bound because $\gamma$ and $\alpha$ were not independently estimated from the experiments.
\end{itemize}

\subsection{Conclusion}

Under the stated cumulative-influence model, embedding optimization is represented by an amplification factor that reduces the modeled poisoning ratio required at equal ASR. The empirical poisoning-efficiency estimates support the same directional conclusion, while the numerical comparison remains conditional on the analytical assumptions.

\printcredits

\bibliographystyle{cas-model2-names}


\bibliography{references_R2}

\end{document}